\documentclass{article}
\usepackage{graphicx} 
\usepackage[margin=1in]{geometry}

\usepackage{xcolor}

\usepackage{dsfont}
\usepackage{lineno}
\usepackage{wrapfig}
\usepackage{lmodern}
\usepackage[utf8]{inputenc} 
\usepackage[T1]{fontenc}        
\usepackage{url}            
\usepackage{booktabs, multirow, array, caption}       
\usepackage{amsfonts}       
\usepackage{nicefrac}     
\usepackage[tracking=true]{microtype}     
\usepackage{xcolor}     
\usepackage{subcaption}
\usepackage{bbm}
\usepackage{mathtools}
\usepackage{amsthm}
\usepackage{enumitem}
\usepackage{natbib}
\usepackage{algorithmic}
\usepackage{colortbl}

\definecolor{rpurple}{RGB}{120, 50, 180}

\usepackage{algorithm}

\newcommand{\diff}{{\rm d}}

\usepackage[colorlinks,linkcolor=orange, anchorcolor=blue, citecolor=blue]{hyperref}
\theoremstyle{definition}
\newtheorem{theorem}{Theorem}[section]

\newtheorem{lemma}[theorem]{Lemma}

\newtheorem{definition}[theorem]{Definition}
\newtheorem{assumption}[theorem]{Assumption}

\newcommand{\pathmeasure}{\PP}
\newcommand{\pathdensity}{p}

\renewcommand{\algorithmicrequire}{\textbf{Input:}}
\renewcommand{\algorithmicensure}{\textbf{Output:}}
\usepackage{smile}
\allowdisplaybreaks[4]
\definecolor{lightpurple}{RGB}{170,140,255}

\definecolor{lightgray}{gray}{0.65} 

\newcommand{\algname}{\texttt{DOIT}}

\title{%
\begingroup
\microtypesetup{tracking=true}
\textls[-10]{Training-Free Adaptation of Diffusion Models via Doob's $h$-Transform}%
\endgroup
}
\author{
  Qijie Zhu\footnotemark[2]
  \footnotemark[1]
  \and
  Zeqi Ye\footnotemark[3]
  \footnotemark[1]
  \and
  Han Liu\footnotemark[2]
  \footnotemark[4]
  \and
  Zhaoran Wang\footnotemark[3]
  \and
  Minshuo Chen\footnotemark[3]
}
\date{}
\begin{document}
\renewcommand{\thefootnote}{\fnsymbol{footnote}}
\maketitle

\footnotetext[1]{Equal contribution.}
\footnotetext[2]{Department of Statistics and Data Science, Northwestern University. \texttt{qijiezhu2029@u.northwestern.edu}}

\footnotetext[3]{Department of Industrial Engineering and Management Sciences, Northwestern University.
\texttt{zeqiye2029@u.northwestern.edu, zhaoran.wang@u.northwestern.edu, minshuo.chen@northwestern.edu}}
\footnotetext[4]{Department of Computer Science, Northwestern University. \texttt{hanliu@northwestern.edu}}

\begin{abstract}
\noindent Adaptation methods have been a workhorse for unlocking the transformative power of pre-trained diffusion models in diverse applications. Existing approaches often abstract adaptation objectives as a reward function and steer diffusion models to generate high-reward samples. However, these approaches can incur high computational overhead due to additional training, or rely on stringent assumptions on the reward such as differentiability. Moreover, despite their empirical success, theoretical justification and guarantees are seldom established. In this paper, we propose \algname{} (\textbf{D}oob-\textbf{O}riented \textbf{I}nference-time \textbf{T}ransformation), a training-free and computationally efficient adaptation method that applies to generic, non-differentiable rewards. The key framework underlying our method is a measure transport formulation that seeks to transport the pre-trained generative distribution to a high-reward target distribution. We leverage Doob's $h$-transform to realize this transport, which induces a dynamic correction to the diffusion sampling process and enables efficient simulation-based computation without modifying the pre-trained model. Theoretically, we establish a high probability convergence guarantee to the target high-reward distribution via characterizing the approximation error in the dynamic Doob's correction. Empirically, on D4RL offline RL benchmarks, our method consistently outperforms state-of-the-art baselines while preserving sampling efficiency. Code: \url{https://github.com/liamyzq/Doob_training_free_adaptation}.
\end{abstract}

\section{Introduction}
Diffusion models have recently become a leading class of generative models, achieving state-of-the-art performance across a wide range of applications, including image generation \citep{song2019generative,song2020denoising,song2020score,ho2020denoising,kong2020diffwave,jeong2021diff,mittal2021symbolic,huang2022prodiff,ulhaq2022efficient,avrahami2022blended}, molecular design \citep{weiss2023guided,guo2024diffusion}, and robotics \citep{chi2025diffusion,reuss2023goal,scheikl2024movement,hou2024diffusion,dasari2025ingredients}. Notably, pre-trained diffusion models already exhibit strong capabilities. They can generate high-fidelity, photorealistic images \citep{song2019generative,ho2020denoising,song2020denoising,song2020score,nichol2021glide,yang2024mastering}, and in robotics, imitation-trained diffusion policies can produce reasonable action sequences for manipulation \citep{chi2025diffusion,ze20243d,scheikl2024movement}.

Despite the tremendous success, there is a pressing need to adapt pre-trained models to specific downstream tasks. For example, in robotics, diffusion models can serve as powerful and flexible policy classes. When trained on offline expert demonstrations, they can closely mimic expert behavior by generating action sequences consistent with the demonstrations~\citep{chi2025diffusion}, yet still underperform on the downstream task objectives that require action generation beyond imitation \citep{ren2024diffusion,ada2024diffusion}. These downstream task objectives can often be summarized as an abstract scalar-valued reward function $r(x)$ on the generated data, where a high reward value is desired. For instance, in Reinforcement Learning (RL) and robotics, the reward $r$ is often tied to task completion (e.g., reaching a target position). Given the reward function, adapting a pre-trained diffusion model amounts to steering its sample generation toward high rewards. 

Many methods have been developed to adapt a pre-trained diffusion model toward high rewards. An ideal adaptation algorithm should be computationally lightweight, not data-hungry in terms of additional samples or interactions, and admit meaningful performance guarantees. One line of work focuses on training-based methods, including RL-based fine-tuning~\citep{clark2023directly,black2023training,fan2023dpok,prabhudesai2023aligning,uehara2024understanding,uehara2024fine,ren2024diffusion,hu2025towards}, guidance-based methods~\citep{dhariwal2021diffusion} and preference-based fine-tuning such as Direct Preference Optimization (DPO) \citep{wallace2024diffusion,yang2024using,lee2025calibrated}. These methods substantially improve reward-aligned performance, at the cost of additional network training and hyperparameter tuning.

A complementary line of work develops training-free methods, which require no additional training, and can still achieve competitive performance. Some approaches assume access to the gradient of the reward \citep{chung2022diffusion,bansal2023universal,yu2023freedom,ye2024tfg,nguyen2025h}. This assumption, however, often fails in practice---for example, the reward function in molecular design often comes from black-box tools or discrete descriptors, which are non-differentiable \citep{trott2010autodock,abramson2024accurate}.
To address this, several methods rely only on querying reward values, which include Sequential Monte Carlo (SMC)-based methods~\citep{trippe2022diffusion,wu2023practical,dou2024diffusion,cardoso2023monte,phillips2024particle,kim2025test} and search-based methods~\citep{li2024derivative,ma2025inference,li2025dynamic,zhang2025vfscale,jain2025diffusion,zhang2025inference}. While avoiding additional network training, these methods increase inference-time cost and are prone to sample collapse \citep{browne2012survey,uehara2025inference}. Such limitations motivate the following key questions:

\vspace{-0.25em}
\begin{center}
\it Can we design an inference-time training-free efficient adaption algorithm for non-differentiable reward?\\
\it If so, can we prove theoretical guarantees of the algorithm?
\end{center}
\vspace{-0.25em}
We provide positive answers to these two questions and present an inference-time adaptation algorithm \algname{} (\textbf{D}oob-\textbf{O}riented \textbf{I}nference-time \textbf{T}ransformation). Specifically, suppose that a pre-trained diffusion model yields a generative distribution $P_{\theta}$, where $\theta$ denotes the pre-trained parameters in the score network. We formulate the adaptation task as sampling from a conditional distribution $P_\theta(\cdot | \cE)$, where the event $\cE$ encapsulates desired conditions, e.g., the reward of generated samples is beyond a threshold $r_0$. We leverage Doob's $h$-transform \citep{rogers2000diffusions,sarkka2019applied} for the measure transport from $P_\theta$ to $P_{\theta}(\cdot | \cE)$. Importantly, Doob's $h$-transform introduces an additive correction term to the dynamic process of sample generation in diffusion models, allowing efficient implementation by keeping the pre-trained parameter $\theta$ frozen. We summarize our methodological and theoretical contributions as follows.

\noindent $\bullet$ Methodologically, we propose simulation-based approximation to the additive correction term in Doob's $h$-transform. Further, we propose \algname{} algorithm (Algorithm~\ref{alg:DOIT_practical}), which is completely training-free, applies to non-differentiable reward functions, and maintains efficient sampling comparable to pre-trained models.

\noindent $\bullet$ Theoretically, we provide a high-probability convergence guarantee for \algname{}. Our analysis characterizes the error stemming from the approximation of the additive correction term in Lemma~\ref{thm:mc_estimation}, and then we propagate the approximation error to establish an end-to-end total-variation bound between the output distribution of \algname{} and the reward-induced target distribution in Theorem~\ref{thm:tv_decomposition}.
    
\noindent $\bullet$ Empirically, we demonstrate that \algname{} effectively steers the distribution of generated samples toward high-reward regions in Section~\ref{sec:expeiment}. On offline RL benchmarks, \algname~consistently outperforms state-of-the-art baselines while maintaining competitive sampling efficiency.

\section{Related Work}
\label{sec:related_work}
\paragraph{Training-Based Adaptation Methods} 
Reward adaptation is often achieved by updating diffusion model parameters with additional training. One approach casts the denoising process as a Markov decision process and applies RL-based methods to fine-tune pre-trained models for optimizing the reward \citep{clark2023directly,black2023training,fan2023dpok,prabhudesai2023aligning,uehara2024understanding,uehara2024fine,ren2024diffusion,hu2025towards}. Guidance-based methods learn a guidance term dependent on the reward function and distill the guidance into pre-trained model parameters~\citep{ho2022classifier,zhang2023adding,yuan2023reward,zhao2024adding}. More recently, preference-based methods such as DPO adapt diffusion models using pairwise comparisons~\citep{wallace2024diffusion,yang2024using,lee2025calibrated}. Despite strong empirical gains, these methods typically incur nontrivial training cost and may require additional data, interaction, or careful hyperparameter tuning.

\paragraph{Training-Free Adaptation Methods}
Training-free adaptation methods steer a pre-trained diffusion model at inference time.
Several methods utilize the gradient of the reward function to guide the denoising sample generation process toward higher reward \citep{chung2022score,bansal2023universal,he2023manifold,yu2023freedom,ye2024tfg,nguyen2025h}.
For non-differentiable reward functions, a growing set of approaches relies on inference-time scaling. A representative class is SMC-based methods \citep{trippe2022diffusion,wu2023practical,dou2024diffusion,kim2025test,singhal2025general}. They maintain a set of particles, reweight them using reward information, and resample to concentrate on high-reward regions. Another class is search-based adaptation methods \citep{li2024derivative,ma2025inference,li2025dynamic}. Many of them perform local search by generating multiple denoising trajectories at each step and selecting the best one. More recent variants combine tree search with local search methods \citep{zhang2025vfscale,jain2025diffusion,zhang2025inference}.

\paragraph{Doob's $h$-transform Based Adaptation Methods}
There are several recent works leveraging Doob’s $h$-transform to steer diffusion models toward a reward-induced target distribution. 
These works modify the sampling process by a correction term, which is learned by additionally training a neural network \citep{denker2024deft,denker2025iterative,chang2026inference}. A training-free method is developed in \citet{nguyen2025h} for text-to-image editing. Yet, it requires a differentiable reward function. Our method is training-free and applies to generic, non-differentiable rewards. Moreover, we establish a theoretical convergence guarantee of our method, which is highly limited with only very recent advances \citep{guo2026conditional,chang2026inference}.

\paragraph{Notation} For a vector $x$, let $\|x\|_2$ denote its Euclidean norm. Let $\|x\|_1$ and $\|x\|_\infty$ denote its
$\ell_1$-norm and $\ell_\infty$-norm, respectively. For a matrix $A$, let $\|A\|_2$ denote its spectral norm. 
We use $\mathcal{O}(\cdot)$ to hide multiplicative constants in upper bounds. Unless otherwise stated, $\nabla$ denotes the gradient with respect to $x$; when there is no ambiguity, we drop $x$ from the notation.
\section{Diffusion Model and Doob's $h$-transform}
\label{sec:diff_doob}
We briefly review the continuous-time formulation of diffusion models and their induced discrete-time sampling SDE (Section~\ref{sec:dm_basics}), and then introduce Doob's $h$-transform (Section~\ref{sec:doob}) that will be used throughout the paper.

\subsection{Diffusion Model Basics}\label{sec:dm_basics}
A diffusion model aims to learn and sample from an unknown
data distribution $P_{\rm data}$ by estimating the score function~\citep{song2019generative,ho2020denoising,song2020denoising}. It consists of coupled forward and backward processes. The forward process is governed by the SDE:
\begin{align}\label{eq:forward_sde}
\diff Y_t = -\frac{1}{2} Y_t \diff t + \diff W_t \quad t \in [0, T], \quad Y_0 \sim P_{\rm data},
\end{align}
where $T$ is a terminal time, and $W_t$ is a Wiener process. We denote $P_t$ as the marginal distribution of $Y_t$ with density $p_t$. The backward process reverses the evolution in the forward process---referred to as denoising for new sample generation. Formally, the backward process is written as the reverse-time SDE:
\begin{align}\label{eq:backward_sde}
\diff X_t
&= \Big[-\tfrac{1}{2}X_t - \nabla \log p_t(X_t)\Big]\diff t + \diff \overline{W}_t,
\end{align}
where $t \in [0, T], X_T \sim P_T $,
and the SDE evolves backward in time from $T$ to $0$. Here, $\overline{W}_t$ denotes an independent Wiener process, and $\nabla \log p_t$ is the score function. As a result, we have $X_0\sim P_{\mathrm{data}}$.

In practice, the reverse-time SDE~\eqref{eq:backward_sde} is intractable due to the unknown terminal distribution $P_T$ and the unknown score function $\nabla_{x}\log p_t(x)$. Following standard practice, we replace $P_T$ by $\mathcal{N}(0,I)$ and $\nabla \log p_t$ by a trained score network $s_\theta(x,t)$. This leads to the following SDE:
\begin{align}
\diff \tilde{X}_t
= \left[-\tfrac{1}{2}\tilde{X}_t - s_\theta(\tilde{X}_t,t)\right]\diff t
+ \diff \overline{W}_t.
\label{eq:practical_reverse_sde}
\end{align}

To simulate process \eqref{eq:practical_reverse_sde}, we discretize the time interval $[0,T]$ using a grid
\(
0=t_0<t_1<\cdots<t_{L-1}<t_L=T,
\)
where $L$ is the number of discretization steps. Over each interval $t\in[t_{l-1},t_l]$, we consider the piecewise SDE:
\begin{align}
\diff \bar{X}_t
=
\left[-\tfrac{1}{2}\bar{X}_{t} - s_\theta(\bar{X}_{t_l},t_l)\right]\diff t
+\diff \overline{W}_t,
\label{eq:piecewise_sde}
\end{align}
which admits an analytical solution and allows for efficient sampling. Specifically, the transition distribution from $\bar X_{t_l}$ to $\bar X_{t_{l-1}}$ in \eqref{eq:piecewise_sde} is
\begin{align}
\label{eq:diffusion_posterior}
\bar X_{t_{l-1}}| \bar X_{t_l} = x_{t_l} \sim \mathcal{N}\left(\mu_{t_l}(x_{t_l},s_\theta),\sigma^2_{t_l} I \right),
\end{align}
where $\mu_{t_l}$ is a linear function of $s_\theta$ and $\sigma^2_{t_l}$ is a noise schedule. To enable fast sampling, existing literature modifies $\mu_{t_l}$ and $\sigma_{t_l}^2$, such as DDIM~\citep{song2020denoising} and Euler ancestral sampling~\citep{karras2022elucidating}; see Appendix~\ref{app:backward_kernel} for explicit forms. We denote $\bar{X}_0\sim P_{\theta}$ as the generated distribution of \eqref{eq:piecewise_sde}. Furthermore, we denote $\bar{X}_t \sim P_{\theta,t}$ and let $\pathdensity_{\theta,t}$ be its marginal density.

\subsection{Doob's $h$-Transform}\label{sec:doob}

Recall that we view adapting a pre-trained diffusion model as steering sample generation
toward high rewards. However, simply performing reward maximization can lead to reward hacking, where generated samples suffer from severe drops in fidelity and diversity \citep{skalse2022defining}. To mitigate this, it is critical to model the conditional data distribution rather than pursuing pure reward maximization. Formally, we aim to sample from the conditional generated distribution $P_{\theta}(\cdot | \mathcal{E}_{\bar{X}_0})$, where $\mathcal{E}_{\bar{X}_0}$ describes the conditions on samples from the terminal distribution, with $\PP(\mathcal{E}_{\bar{X}_0})>0$. A commonly used choice of $\cE_{\bar{X}_0}$ is to select high-reward samples as
\[
\mathcal{E}_{\bar{X}_0} =\{\bar{X}_0: r({\bar{X}_0})\geq r_0 \},
\]
where $r_0$ is a threshold. Note that $P_{\theta}(\cdot | \cE_{\bar{X}_0})$ is an approximation to the ideal conditional data distribution $P_{\rm data}(\cdot | \mathcal{E}_{X_0})$ with $X_0$ the terminal state of \eqref{eq:backward_sde}, where $\cE_{X_0}$ is defined analogously by replacing $\bar X_0$ with $X_0$.

\paragraph{Transport $P_\theta$ to $P_{\theta}(\cdot | \cE_{\bar{X}_0})$.} Doob's $h$-transform \citep{rogers2000diffusions,sarkka2019applied} provides a principled probabilistic framework to modify the generated distribution $P_{\theta}$ towards the desired conditional distribution $P_{\theta}(\cdot|\mathcal{E}_{\bar{X}_0})$. The key is to define a Doob's $h$-function that dynamically modifies \eqref{eq:piecewise_sde}.
\begin{definition}[Doob's $h$-function]
\label{def:h-function}
The Doob's $h$-function for $0\leq t \leq T$ is defined as
\begin{align*}
    h(x_t,t)& = \PP(\mathcal{E}_{\bar{X}_0} | \bar{X}_t = x_t),
\end{align*}
where $\PP$ is taken with respect to the randomness in \eqref{eq:piecewise_sde} and the randomness in $\mathcal{E}_{\bar{X}_0}$.
\end{definition}
By Bayes' rule, the Doob's $h$-function induces a tilted density at time $t$:
\begin{align}
p^{h}_{\theta,t}(x_t)=p_{\theta,t}(x_t|\mathcal{E}_{\bar{X}_0})
    = h(x_t,t) p_{\theta,t}(x_{t}) / \PP(\mathcal{E}_{\bar{X}_0}).\label{eq:reweight_h}
\end{align}
In particular,
$p_{\theta,0}^h(x)=p_{\theta,0}\left(x|\mathcal{E}_{\bar{X}_0}\right)$, thus realizing the modification from $P_{\theta}$ towards the desired conditional distribution $P_{\theta}(\cdot|\mathcal{E}_{\bar{X}_0})$.

Doob's $h$-transform is a powerful tool for adapting the generated distribution to a generic target distribution via the reweighting induced by the $h$-function. To formalize this capability, we establish the following lemma.
\begin{lemma}
\label{lem:bounded_tilting_via_event}
Let $q$ be the density function of the target distribution such that
\[
\left\|q / p_{\theta,0}\right\|_\infty \le C_q \quad \text{for a constant} \quad C_q < \infty.
\]
Let $U\sim\mathrm{Unif}(0,1)$ be independent of $\bar X_0$. Setting $h(x_t, t) = \PP(\mathcal{E}_{\bar{X}_0} | \bar{X}_t = x_t)$ with
\[
\mathcal{E}_{\bar X_0}=\left\{U\le C_q^{-1} q(\bar X_0) / p_{\theta,0}(\bar X_0)\right\}
\]
leads to $p^h_{\theta,0}(x)\ =q(x).$
\end{lemma}
The proof of Lemma~\ref{lem:bounded_tilting_via_event} is deferred to Appendix~\ref{append:proof_lem_bounded_tilting_via_event}.

\paragraph{Tilted Sampling Process for $P_\theta(\cdot | \cE_{\bar{X}_0})$.} A vital advantage of Doob's $h$-transform is that sampling from the tilted distribution in \eqref{eq:reweight_h} reduces to adding a time-dependent correction term to \eqref{eq:piecewise_sde}. Precisely, in order to generate samples from $P_\theta(\cdot | \cE_{\bar{X}_0})$, we simulate the following piecewise SDE that evolves backward in time from $T$ to $0$ \citep{rogers2000diffusions,sarkka2019applied,nguyen2025h},
\begin{align}
\diff \bar{X}_t^h
\!=\!
\left[-\tfrac{1}{2}\bar{X}_{t}^h - s_{\theta}(\bar{X}_{t_l}^h,t_l)-\nabla \log h(\bar{X}_{t}^h,t)\right]\diff t +\diff \overline{W}_t,
\label{eq:piecewise_sde_corrected}
\end{align}
where $t\in [t_{l-1},t_l]$. We refer to $\{\bar X_t^h\}_{t\in[0,T]}$ as the tilted sampling process.  
We denote $\bar X_t^h\sim P_{\theta,t}^h$ as its time-$t$ marginal distribution, and $p_{\theta,t}^h$ as its marginal density. We use $\PP_\theta^h$ to denote the corresponding path measure. 

As can be seen, Doob's $h$-transform modifies the original score network $s_{\theta}$ by adding a dynamic Doob’s correction term $\nabla \log h$. 
In practice, we apply a piecewise constant approximation to $\nabla \log h$, replacing $\nabla \log h(\bar X_t^h,t)$ by $\nabla \log h(\bar X_{t_l}^h,t_l)$ for $t\in[t_{l-1},t_l]$. Consequently, we simulate the following piecewise SDE:
\begin{align}
\diff \hat{X}_t^h
\!=\!
\left[\!-\tfrac{1}{2}\hat{X}_{t}^h - s_{\theta}(\hat{X}_{t_l}^h,t_l)-\nabla \log h(\hat{X}_{t_l}^h,t_l)\right]\!\diff t +\diff \overline{W}_t,
\label{eq:piecewise_sde_corrected_piecewise}
\end{align}
where $t\in [t_{l-1},t_l]$.

However, as exactly evaluating $\nabla \log h(x,{t_l})$ is intractable, we must construct a reliable approximation to simulate the process $\{\hat{X}_t^h\}_{t\in[0,T]}$.
\section{Adaptation Algorithm Based on Doob's $h$-Transform}\label{sec:doit}
This section presents the \algname~({\bf D}oob-{\bf O}riented {\bf I}nference-time {\bf T}ransformation) algorithm, solving the major challenge of the intractability of exactly computing $\nabla \log h$. We develop a simulation-based approximation to $\nabla \log h$ and summarize our algorithm in Algorithm~\ref{alg:doit_pseudocode}.

To derive a practical approximation of $\nabla \log h$, we write $\nabla\log h=\nabla h/h$, and then we approximate the numerator $\nabla h$ and the denominator $h$ separately. By Definition~\ref{def:h-function}, the $h$-function can be straightforwardly approximated by Monte Carlo (MC) samples. Therefore, we focus on the more intricate term $\nabla h$. The following lemma expresses $\nabla h$ as a single expectation that is suitable for approximation.

\begin{lemma}\label{lem:score_h_gradient}
Fix a discretization index $l\in\{1,\dots,L\}$ and any $x$. Assume $s_{\theta}(x,t)$ is continuously differentiable with respect to $x$. Then it holds that
\begin{align*}
\nabla h(x,t_l)
\!=\!
\EE\!\left[
h(\bar X_0,0)
\nabla_{\bar X_{t_l}}\!\log \phi_{\theta}(\bar X_{t_{l-1}}\!| \bar X_{t_l})
\Big|\bar X_{t_l}=x
\right]\!,
\end{align*}
where $\phi_{\theta}$ is the Gaussian density defined in \eqref{eq:diffusion_posterior}, and the expectation is taken over the conditional distribution of the backward trajectory
$(\bar X_{t_{l-1}},\dots,\bar X_0)| \bar X_{t_l}=x$.
\end{lemma}
The proof of Lemma~\ref{lem:score_h_gradient} is deferred to Appendix~\ref{app:proof_lem_score_h_gradient}.

Lemma~\ref{lem:score_h_gradient} suggests using a sample average to approximate the expectation and further approximate $\nabla h$. More importantly, since the gradient is only taken over the Gaussian transition density $\phi_\theta$, approximating $\nabla h$ does not require differentiability in the reward function. 


\paragraph{Sample Average Approximation to $\nabla h$ and $h$.} At time $t_l$, given a state $x_{t_l}$, we simulate $M$ trajectories $\{x_{t_{l-1}}^{(m)}, \dots, x_0^{(m)}\}_{m=1}^M$ using the transition kernel in \eqref{eq:diffusion_posterior}. Then $\nabla h$ and $h$ are approximated respectively by
\begin{align*}
\nabla \hat{h}(x_{t_l},t_l)&\textstyle =\!\frac{1}{M}\sum_{m=1}^{M}h(x_0^{(m)},0)\nabla_{x_{t_l}}\log \phi_{\theta}(x_{t_{l-1}}^{(m)}| x_{t_l}),\\
\hat h(x_{t_l},t_l)&\textstyle    =\!\frac{1}{M}\sum_{m=1}^{M} h(x_0^{(m)},0).
\end{align*}
We approximate $\nabla \log h$ via a plug-in ratio approximation,
\begin{align}
    \nabla \log \hat{h}(x_{t_l},t_l)=\frac{\nabla \hat{h}(x_{t_l},t_l)}{\hat h(x_{t_l},t_l)\vee \eta_{t_l}}.\label{eq:mc_estimator}
\end{align}
Here, we introduce a truncation level $\eta_{t_l}>0$ for numerical stability. When $h(x_{t_l}, t_l)$ is small, its approximation $\hat{h}$ will be close to zero. Na\"{i}vely using $\hat{h}$ causes $\nabla \log \hat{h}$ to have an exploding magnitude. We remark that such a stability issue is intrinsic to the quality of the pre-trained model, rather than a limitation of our proposed method. A small $h(x_{t_l}, t_l)$ indicates that in the pre-trained generative distribution, the desired condition is hardly satisfied by generated samples, incurring a significant gap between $P_\theta$ and $P_{\theta}(\cdot | \cE_{\bar{X}_0})$.

We remark that our derived approximation of $\nabla\log h$ is different from existing approaches \citep{bansal2023universal,nguyen2025h}. They often pass through the expectation to approximate
\begin{align*}
     \EE[h(\bar X_0,0)|\bar X_{t_l}=x_{t_l}] \approx  h(\EE[\bar X_0|\bar X_{t_l}=x_{t_l}],0).
\end{align*}
When taking derivatives with respect to $x_{t_{l}}$, such an approximation inevitably requires the differentiability of $h(\cdot, 0)$---consequently the differentiability of the reward function. In contrast, our approximation unrolls the transition from $\bar{X}_{t_l}$ to $\bar{X}_0$ into a series of incremental Gaussian transitions, i.e., $\phi_\theta$. This allows us to pass through the differentiation directly to the transition probability instead of requiring differentiability of the reward function.

Given $\nabla \log \hat{h}$, we simulate samples using \eqref{eq:piecewise_sde_corrected_piecewise}, achieving the adaptation from $P_\theta$ to $P_{\theta}(\cdot | \cE_{\bar{X}_0})$. A prototypical algorithm is summarized below.

\begin{algorithm}[h]
\caption{\algname{}: Prototypical version}
\label{alg:doit_pseudocode}
\begin{algorithmic}[1]
\REQUIRE Pre-trained score $s_\theta(x,t)$; the number of MC samples  $M$; truncation level $\{\eta_{t_l}\}_{l=1}^L$.

\STATE Sample $\hat{X}_{t_L}^h\sim \mathcal N(0,I)$. 
\FOR{$l=L,L-1,\dots,1$}
    \STATE Compute $\nabla \log \hat h(\hat{X}_{t_l}^h,t_l)$ in \eqref{eq:mc_estimator} using $M$ backward rollouts
    initiated at $\hat{X}_{t_l}^h$ by simulating~\eqref{eq:piecewise_sde_corrected_piecewise}.
    \smallskip
    \STATE $\nabla \log \hat{p}_{\theta,t_l}^h(\hat{X}_{t_l}^h) \leftarrow s_\theta(\hat{X}_{t_l}^h,t_l)+\,\nabla \log \hat h(\hat{X}_{t_l}^h,t_l)$.
    \smallskip
    \STATE Sample $\hat{X}_{t_{l-1}}^h\sim \mathcal N\!\big(\mu_{t_l}(\hat{X}_{t_l}^h,\nabla \log \hat{p}_{\theta,t_l}^h),\,\sigma_{t_l}^2 I\big)$.
\ENDFOR
\STATE \textbf{output} $\hat{X}_0^h\sim \hat{P}_{\theta}^h$.
\end{algorithmic}
\end{algorithm}

Here $\hat{P}^h_{\theta}$ denotes the distribution of the generated sample $\hat X_0^h$. Each iteration in Algorithm~\ref{alg:doit_pseudocode} relies on simulating $M$ independent backward sampling trajectories using \eqref{eq:piecewise_sde}, which is computationally expensive. To mitigate the computational overhead, we significantly reduce the simulation cost by predicting clean data directly using the score network, while skipping all intermediate states in sampling trajectories. Algorithm~\ref{alg:DOIT_practical} in Section~\ref{sec:practical_ver} summarizes such computational modifications and demonstrates strong performance over state-of-the-art baselines in offline RL benchmarks.

\begin{figure}[H]
    \centering
    \includegraphics[width=0.65\linewidth]{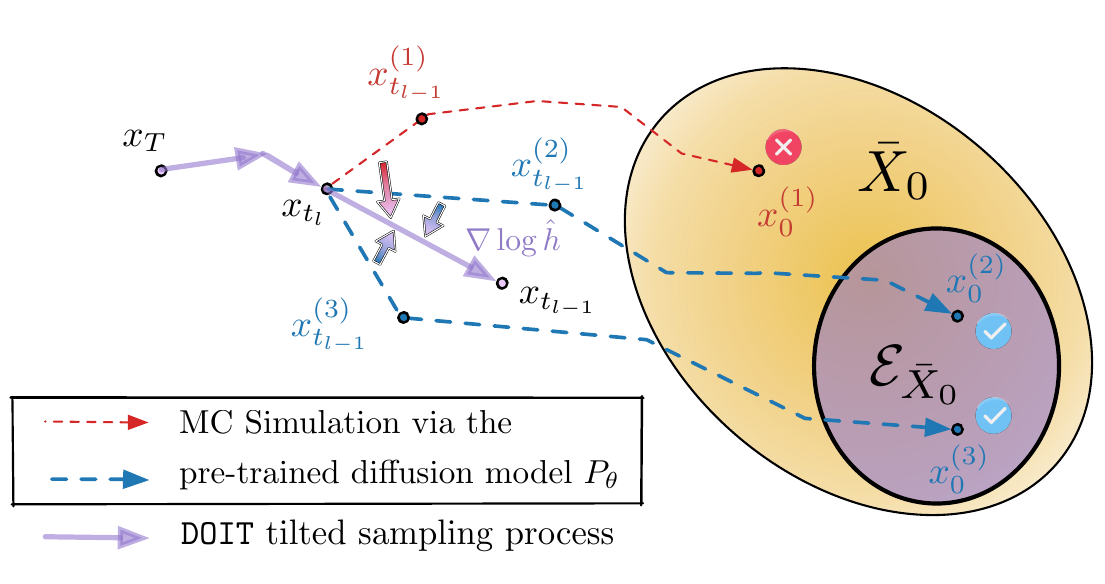}
    \caption{\algname: At each $t_l$, we simulate $M$ trajectories (here, $M=3$) starting from $x_{t_l}$ to approximate $\nabla\log h(x_{t_l},t_l)$ via~\eqref{eq:mc_estimator}, then utilize it to modify the sampling dynamics.}
    \label{fig:doit-demo}
\end{figure}

\section{Convergence Guarantee of \algname}
In this section, we provide a convergence guarantee of $\hat{P}^h_\theta$ generated by \algname{} relative to the ideal conditional data distribution $P_{\mathrm{data}}(\cdot| \cE_{X_0})$. Our analysis proceeds in two steps. We first quantify the approximation error of $\nabla\log h(x_{t_l},t_l)$. We then propagate this approximation error through the tilted sampling dynamics in \eqref{eq:piecewise_sde_corrected} to obtain an end-to-end distribution approximation bound.

To begin with, we impose the following assumption on the pre-trained generated distribution $P_{\theta}$ and score $s_{\theta}(x_{t_l},t_l)$.
\begin{assumption}\label{assump:bound_h_0}
There exist constants $\rho\in(0,1]$ and $G>0$ such that $\mathbb{P}(\mathcal{E}_{\bar{X}_0})
\ge\rho,$
and for any index $l\in\{1,\dots,L\}$, it holds that $\sup_{x}\left\|\nabla_x s_{\theta}(x,t_l)\right\|_2 \le G$.
\end{assumption}
The first condition is a non-degeneracy assumption, it requires the target event $\cE_{\bar X_0}$ to occur with at least probability $\rho$ under $P_\theta$. This lower bound is essential for the stability of the MC approximation in \eqref{eq:mc_estimator}. When $\rho$ is extremely small, accurately approximating $\nabla \log h$ from a finite number of rollouts is fundamentally difficult, as $h(x_0^{(m)},0)$ is not close to zero for only a small fraction of trajectories.

The second condition imposes a global Lipschitz-type regularity on the pre-trained score network, which is a key ingredient for establishing concentration of the MC approximation. Such smoothness assumptions are reasonable in practice, since modern training typically incorporates regularization and stabilization techniques that implicitly control the network’s Lipschitz
continuity, such as weight decay \citep{krogh1991simple,loshchilov2017decoupled}.

The following lemma establishes a high probability bound on approximating $\nabla \log h$.
\begin{lemma}[Approximation bound of $\nabla \log h$]\label{thm:mc_estimation}
Suppose Assumption~\ref{assump:bound_h_0} holds. Fix a discretization index $l\in\{1,\dots,L\}$ and $\delta\in(0,1)$. Then for a sufficiently large number of MC samples $M$ and $\eta_{t_l}=M^{-1/6}$, with probability at least $1-\delta $, it holds that
\begin{align*}
    \EE_{X_{t_l}\sim P^h_{\theta,t_l}}\left[ \norm{\nabla \log \hat{h}(X_{t_l},t_l)-\nabla \log h(X_{t_l},t_l)}_2^2 \right]=~\cO\left(\frac{1}{\sigma_{t_l}^2}\frac{ \log \frac{M}{\delta}\sqrt{\log(1/\delta)}}{M^{1/6}}\right)&.
\end{align*}
\end{lemma}
The proof of Lemma~\ref{thm:mc_estimation} is deferred to Appendix~\ref{app:proof_mc_estimation}. 

Lemma~\ref{thm:mc_estimation} gives a high-probability bound on the mean-squared error of the MC approximation for $\nabla\log h(\cdot,t_l)$. The error decays as the number of MC samples $M$ increases,  with a dimension-independent rate $\cO(M^{-1/6})$, reflecting the concentration of the MC approximation \eqref{eq:mc_estimator}. 

Moreover, the bound depends on $\sigma_{t_l}^2$, indicating that approximating $\nabla \log h$ becomes more challenging when $\sigma_{t_l}$ is small and $t_l$ is near zero. In this regime, the conditional transition $X_{t_{l-1}}| X_{t_l}$ is nearly deterministic, meaning the transition density $\phi_\theta(\cdot | X_{t_l})$ approaches a Dirac delta function. This results in an ill-behaved score function and causes the magnitude of $\nabla\log h(\cdot,t_l)$ to become extremely large. Consequently, even small approximation errors in the numerator are significantly amplified. The remaining factor $\log(M/\delta)\sqrt{\log(1/\delta)}$ comes from concentration over the MC randomness under a $1-\delta$ guarantee.

We propagate the MC error bound in Lemma~\ref{thm:mc_estimation} through the tilted sampling process, which yields an end-to-end convergence guarantee by bounding the discrepancy between the target conditional distribution $P_{\mathrm{data}}(\cdot| \cE_{X_0})$ and the output distribution $\hat P_{\theta}^h$ produced by \algname{}. We begin by introducing an assumption that quantifies the discretization error incurred by \eqref{eq:piecewise_sde_corrected_piecewise}.
\begin{assumption}\label{assump:dis_error}
For any $l\in \{1,\dots,L\}$, and for any $t\in [t_{l-1},t_l]$,
we assume that the discretization error is uniformly bounded by $\varepsilon_{\mathrm{dis}}\geq 0$, 
    \begin{align*}
        \EE_{\,\PP_{\theta}^h}\left[\|\nabla \log h(\bar{X}_{t_l}^h,t_l)-\nabla \log h(\bar{X}_{t}^h,t)\|_2^2\right]\leq \varepsilon_{\mathrm{dis}}.
    \end{align*}
\end{assumption}
Under additional regularity conditions (e.g., Lipschitz continuity of $\nabla\log h$ in $(x,t)$ together with mild properties of the event $\cE_{\bar X_0}$), such a bound can be derived explicitly; see \citet{chen2022sampling} for reference.
\begin{theorem}
\label{thm:tv_decomposition}
Suppose Assumptions~\ref{assump:bound_h_0} and~\ref{assump:dis_error} hold and choose $\eta_t$ as in Lemma~\ref{thm:mc_estimation}. With probability at least $1-\delta$, it holds that
\begin{align*}
& \quad \;\mathrm{TV} \left(P_{\mathrm{data}}(\cdot|\cE_{X_0}), \hat{P}_{\theta}^h\right)\lesssim \frac{1}{\rho} \mathrm{TV}(P_{\mathrm{data}}, P_{\theta})+\sqrt{\frac{\kappa_{\sigma} \log \frac{M}{\delta}\sqrt{\log(1/\delta)}}{M^{1/6}}+\varepsilon_{\mathrm{dis}}T},
\end{align*}
where $\kappa_{\sigma}=\frac{T}{L}\sum_{l=1}^L  \sigma_{t_l}^{-2}$.
\end{theorem}
The proof of Theorem~\ref{thm:tv_decomposition} is deferred to Appendix~\ref{append:tv_decomposition}. 

Theorem~\ref{thm:tv_decomposition} decomposes the discrepancy into two terms. The first term, $\frac{1}{\rho}\,\mathrm{TV}(P_{\mathrm{data}},P_\theta)$, consists of two interpretable factors. The multiplier $1/\rho$ reflects the intrinsic difficulty of the conditional generation task. When $\mathbb{P}(\mathcal{E}_{\bar X_0})$ is small, the pre-trained model rarely generates the desired outcomes, thereby posing a significant challenge to achieving high rewards. The second factor, $\mathrm{TV}(P_{\mathrm{data}},P_\theta)$, measures the discrepancy between the generated distribution and the ground truth data distribution. This factor corresponds to the explicit convergence rates of diffusion models established in prior literature \citep{block2020generative,chen2023score,wibisono2024optimal}. Notably, \citet{oko2023diffusion} provided minimax rate guarantees. In our analysis, we treat it as an irreducible error inherited from the pre-training phase.

The second term bounds the total sampling error, combining the MC approximation error and the discretization error. The term $\varepsilon_{\mathrm{dis}}\,T$ accounts for the discretization error introduced by the piecewise-constant approximation of $\nabla\log h$ in \eqref{eq:piecewise_sde_corrected_piecewise}. Moreover, $\kappa_\sigma=\frac{T}{L}\sum_{l=1}^L \sigma_{t_l}^{-2}$ summarizes the accumulation of noise-scale
effects, under the standard DDPM/VP setting \citep{ho2020denoising},
$\kappa_\sigma=\mathcal{O}\left(\log(T/t_1)\right)$.

Together, the theorem implies that when $\mathrm{TV}(P_{\mathrm{data}}, P_{\theta})\to 0$, $\varepsilon_{\mathrm{dis}}\to 0$, and $M$ increases, $\hat{P}_{\theta}^h$ converges to $P_{\mathrm{data}}(\cdot| \cE_{ X_0})$ with high probability in terms of the total variation distance.

\section{Experiments}
\label{sec:expeiment}

In this section, we first instantiate a practically efficient version of \algname{} and detail our specific choice of the $h$-function (Section~\ref{sec:practical_ver}). Subsequently, we validate that \algname{} effectively adapts the reward distribution of generated samples toward higher value regions (Section~\ref{sec:exp_adapt_dist}). Finally, we demonstrate the performance of \algname{} on offline reinforcement learning tasks (Section~\ref{sec:exp_rl}).

\subsection{Practical Instantiation of \algname{}}
\label{sec:practical_ver}

\begin{algorithm}[H]
\caption{\algname: Practical version}
\label{alg:DOIT_practical}
\begin{algorithmic}[1]
\REQUIRE {Pre-trained score $s_\theta(x,t)$; correction strength $\gamma$;
number of MC samples $M$;
 time threshold $l^*$; truncation levels $\{\eta_{t_l}\}_{l=1}^L$.}

\STATE Sample $x_{t_L}$ from $\mathcal N(0,I)$.

\FOR{$l = L, L-1, \dots, 1$}
    \IF{$1 < l \le l^*$}
        \STATE Sample $\{x_{t_{l-1}}^{(m)}\}_{m=1}^M$ from $\mathcal{N}\left(\mu_{t_l}(x_{t_l},s_\theta),\sigma^2_{t_l} I \right)$.
        \STATE Compute $\{\hat{x}_0^{(m)}\}_{m=1}^M$ via~\eqref{eq:x_0_surrogate}.
        \STATE Compute $\nabla \log \hat{h}(x_{t_l},t_l)$ in~\eqref{eq:mc_estimator} via $\{\hat{x}_0^{(m)}\}_{m=1}^M$.
        \STATE $\nabla \log \hat{p}_{\theta,t_l}^h(x_{t_l}) \leftarrow
            s_\theta(x_{t_l},t_l) + \gamma\,\nabla \log \hat{h}(x_{t_l},t_l)$.
        \STATE Sample $x_{t_{l-1}}$ from $ \mathcal{N}(\mu_{t_l}(x_{t_l},\nabla \log \hat{p}_{\theta,t_l}^h), \sigma^2_{t_l} I )$.
    \ELSE
        \STATE Sample $x_{t_{l-1}}$ from $\mathcal{N}\left(\mu_{t_l}(x_{t_l},s_\theta),\sigma^2_{t_l} I \right)$.
    \ENDIF
\ENDFOR

\STATE {\bf Return} $x_0$.
\end{algorithmic}
\end{algorithm}


In the prototypical version of \algname{} (Algorithm~\ref{alg:doit_pseudocode}), approximating $\nabla\log \hat{h}(x_{t_l},t_l)$ at each state $x_{t_l}$ necessitates $M$ full backward rollouts to obtain the terminal states $\{x_0^{(m)}\}_{m=1}^M$. This process incurs a prohibitive computational cost due to the additional Number of Function Evaluations (NFEs) of the score network $s_\theta$. To circumvent this bottleneck, we introduce an efficient surrogate $\hat{x}_0^{(m)}$ to approximate $x_0^{(m)}$.


At state $x_{t_l}$, we first sample $M$ one-step lookahead states $\{x_{t_{l-1}}^{(m)}\}_{m=1}^M$ from the transition kernel~\eqref{eq:diffusion_posterior}. We then approximate the terminal state $x_0^{(m)}$ using Tweedie’s formula~\citep{efron2011tweedie}, which calculates the posterior mean given a noisy state:
\begin{align*}
    \mathbb{E}[x_0 \mid x_{t_{l-1}}^{(m)}] = \frac{x_{t_{l-1}}^{(m)} + (1 - e^{-t_{l-1}}) s_\theta(x_{t_{l-1}}^{(m)}, t_{l-1})}{e^{-t_{l-1}/2}},
\end{align*}
where the coefficients are determined by the forward SDE~\eqref{eq:forward_sde}. Evaluating this expression exactly still requires calling the score network at the new lookahead states. To avoid this, we further reuse the score $s_\theta(x_{t_l}, t_l)$ computed at step $t_l$. Consequently, we define the surrogate $\hat{x}_0^{(m)}$ as:
\begin{align}
\label{eq:x_0_surrogate}
    \hat{x}_0^{(m)}
    = e^{t_{l-1}/2}(x_{t_{l-1}}^{(m)} + (1 - e^{-t_{l-1}}) s_\theta(x_{t_l}, t_l)).
\end{align}
By substituting $x_0^{(m)}$ with $\hat{x}_0^{(m)}$, \algname{} incurs zero additional NFEs of the score network compared to the vanilla generation process. We empirically validate this approximation in Section~\ref{sec:exp_adapt_dist}, demonstrating that it achieves performance comparable to full trajectory simulation.

Building on this surrogate, we present Algorithm~\ref{alg:DOIT_practical}---an efficient implementation of \algname. Notably, to further enhance flexibility and stability, we introduce two hyperparameters: a time threshold $l^*$ and a correction strength $\gamma$, which restrict the application of the Doob correction to specific timesteps and control its magnitude, respectively.

\paragraph{Choice of $h$-function.}
Practical application requires a tailored choice of the $h$-function. In reinforcement learning and energy-based modeling, the target density is typically defined via \textit{exponential tilting} of the base distribution~\citep{peters2010relative,haarnoja2018soft}, taking the form $q(x) \propto p_{\theta,0}(x)\exp(r(x)/\tau)$, where $\tau > 0$ is a temperature parameter controlling the sharpness of the reward distribution: smaller $\tau$ forces a stronger concentration on high-reward regions. Following this paradigm, we specify the $h$-function at the terminal time $t=0$ as
\begin{align}
    \label{eq:h_terminal}
    h(x, 0) \propto \exp\left({r(x)}/{\tau}\right).
\end{align}
This $h$-function corresponds to a particular choice of $\cE_{\bar{X}_0}$, emphasizing high-reward regions. A detailed derivation is deferred to Appendix~\ref{app:event_construction}. We run Algorithm~\ref{alg:DOIT_practical} with the $h$-function in~\eqref{eq:h_terminal} for all subsequent experiments.

\subsection{\algname{} Adapts the Sampling Distribution}
\label{sec:exp_adapt_dist}

We demonstrate that \algname{} adapts the generation process for each sample, effectively tilting the entire reward distribution toward higher-value regions. We conduct experiments by applying \algname{} to \texttt{Stable Diffusion v1.5}~\citep{Rombach_2022_CVPR} to improve rewards of text-to-image generation.


We first use the LAION Aesthetic Score~\citep{beaumont_laion_aesthetic_predictor_2022} as the reward. We generate $K=32$ images and analyze their aesthetic score distributions. To assess performance and robustness, we sweep across the sampling temperature $\tau$ and the correction strength $\gamma$, comparing the resulting reward distributions against the pre-trained model.

 \begin{figure*}[t]
    \centering
    \includegraphics[width=\linewidth]{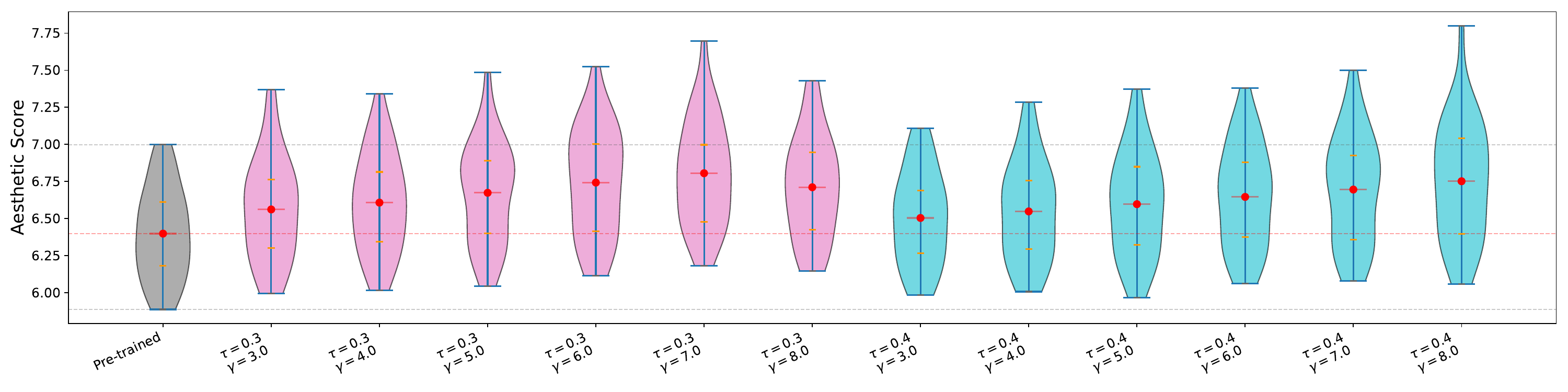}
    \caption{
    Violin plots of aesthetic scores for the samples generated by \texttt{Stable Diffusion v1.5}, comparing the vanilla generation result and applying \algname{} across different $(\tau,\gamma)$ settings.
    The blue bars indicate the minimum and maximum scores, the orange bars represent the first \& third quantiles, and the red marker denotes the mean.
    }
    \label{fig:violin}
\end{figure*}

The results, summarized in Figure~\ref{fig:violin}, demonstrate that our method successfully transports the pre-trained reward distribution to high-reward counterparts; some of the example images are shown in Figure~\ref{fig:demo}. Notably, we identify a trade-off region: lower temperatures (corresponding to sharper reward landscapes) require lower correction strengths to maintain stability, whereas higher temperatures accommodate stronger correction. Within this regime, \algname{} robustly improves aesthetic scores. Experimental details and sensitivity ablation studies on $(\tau,\gamma)$ are provided in Appendix~\ref{app:exp_sd_1.5}.

\paragraph{Comparison of Surrogate and Full Simulation.}
\begin{table}[!ht]
\vspace{-0.5em}
\centering
\small
\caption{Comparison of aesthetic score statistics between the surrogate and full trajectory simulation methods. Runtime (seconds) is reported as the generation time per image, excluding reward evaluation time.}
\begin{tabular}{c|ccc|c}
\hline
 & Min & Mean &  Max & Runtime \\
\hline
Pre-trained
& 5.737 $\pm$ 0.102
& 6.372 $\pm$ 0.031
& 7.052 $\pm$ 0.085
& 1.260 $\pm$ 0.046\\ 
\hline
 Surrogate
& 5.987 $\pm$ 0.180
& 6.726 $\pm$ 0.072
& 7.501 $\pm$ 0.055
& 1.712 $\pm$ 0.036  \\ 
Simulation
& 6.006 $\pm$ 0.099
& 6.714 $\pm$ 0.055
& 7.553 $\pm$ 0.113
& 39.584 $\pm$ 0.142 \\ 
\hline
\end{tabular}
\label{tab:tw_vs_sim}
\end{table}
We compare the performance of the pre-trained model against \algname{} (with $\tau=0.3, \gamma=6.0$) implemented with either the surrogate or full trajectory simulation to approximate the correction term. Table~\ref{tab:tw_vs_sim} shows that the surrogate method achieves reward improvements comparable to the full simulation baseline. Crucially, the surrogate approach requires significantly lower runtime than the full simulation.

\paragraph{Integration with Resampling Based Methods.} As discussed in Section~\ref{sec:related_work}, a distinct line of research focuses on reweighting and resampling strategies. Intuitively, rather than locally correcting the trajectory as in \algname{}, these methods start with multiple candidates and evaluate them at intermediate steps, subsequently reweighting and resampling to favor higher-reward ones. Because these approaches operate on a population level while \algname{} improves the per-sample reward, they are orthogonal and complementary. 

To demonstrate the effect of this integration, we combine \algname{} with two representative strategies: the fundamental Best-of-K (BoK) baseline and the state-of-the-art method BFS~\citep{zhang2025inference}. Following the experimental setup in~\citet{zhang2025inference}, we use ImageReward~\cite{xu2023imagereward} as the reward feedback and compare against several recent reweighting-based baselines, including FK-Steering~\citep{singhal2025general}, DAS~\citep{kim2025test}, TreeG~\citep{guo2025training}, and SVDD~\citep{li2024derivative}. Experimental details are in Appendix~\ref{app:exp_search}.
\begin{table*}[h]
\centering
\small
\caption{Performance comparison on ImageReward. $K$ represents the number of particles. Combining \algname{} with  both BoK and BFS provides a significant boost.}
\label{tab:bon_bfs_doit}
\resizebox{1.0\linewidth}{!}{
\begin{tabular}{c|cccccccc}
\hline
$K$
& BoK
& \cellcolor{gray!30} BoK + \algname{} (Ours)
& FK-Steering
& DAS
& TreeG
& SVDD
& BFS
& \cellcolor{gray!30} BFS + \algname{} (Ours) \\
\hline
4
& $0.702 \pm 0.057$
& \cellcolor{gray!30} $0.875 \pm 0.017$
& $0.743 \pm 0.037$
& $0.878 \pm 0.028$
& $0.860 \pm 0.033$
& $0.667 \pm 0.076$
& $0.882 \pm 0.029$
& \cellcolor{gray!30} $0.950 \pm 0.016$ \\
8
& $0.896 \pm 0.031$
& \cellcolor{gray!30} $1.001 \pm 0.015$
& $0.926 \pm 0.042$
& $1.052 \pm 0.033$
& $1.023 \pm 0.018$
& $0.775 \pm 0.087$
& $1.087 \pm 0.031$
& \cellcolor{gray!30} $1.115 \pm 0.009$ \\
\hline
\end{tabular}}
\end{table*}

The results are presented in Table~\ref{tab:bon_bfs_doit}. We observe that combining \algname{} with both BoK and BFS yields significant performance gains. Because \algname{} tilts the sampling distribution toward high-reward regions, the candidate pools generated for BoK and BFS contain higher-quality samples, thereby improving the final outcome.

\subsection{Performance Evaluation on Offline RL Benchmarks}
\label{sec:exp_rl}

Diffusion policies, utilizing diffusion models as action generators, have recently emerged as a powerful paradigm in robotics~\citep{chi2025diffusion}. However, because they are typically trained on offline datasets, adapting them to inference-time objectives remains a challenge. In this section, we apply \algname{} to the offline reinforcement learning setting, adapting diffusion policies toward higher inference-time returns to demonstrate the effectiveness of our method. We follow the setup of~\citet{lu2023contrastive}, utilizing their pre-trained diffusion policies on the D4RL benchmark~\citep{fu2020d4rl} and employing their Q-functions as reward feedback to guide the generation of action sequences. 
We evaluate performance on standard D4RL locomotion tasks, and compare \algname{} against training-based baselines, including IQL~\citep{kostrikov2021offline}, Diffuser~\citep{janner2022planning}, D-QL~\citep{wang2022diffusion}, and QGPO~\citep{lu2023contrastive}, as well as recent inference-time methods such as TFG~\citep{ye2024tfg}, DAS~\citep{kim2025test}, and TTS~\citep{zhang2025inference}. Implementation details are in Appendix~\ref{app:exp_rl}.


\begin{table*}[h]
\centering
\small
\setlength{\tabcolsep}{6pt}
\caption{Performance comparison of different methods on D4RL locomotion tasks. Training-time methods are marked with $\clubsuit$ and inference-time methods are marked with $\spadesuit$.
For each task, the best-performing inference-time method has its mean highlighted in bold.}
\label{tab:locomotion_colored_best}
\renewcommand{\arraystretch}{1.15}
\resizebox{1.0\linewidth}{!}{\begin{tabular}{llcccc|ccccc}
\toprule
\textbf{Dataset} & \textbf{Environment} &
\textbf{IQL ($\clubsuit$)} &
\textbf{Diffuser ($\clubsuit$)} &
\textbf{D-QL ($\clubsuit$)} &
\textbf{QGPO ($\clubsuit$)} &
\textbf{TFG ($\spadesuit$)} &
\textbf{DAS ($\spadesuit$)} &
\textbf{TTS ($\spadesuit$)} &
\cellcolor{gray!30}
\textbf{\algname{} ($\spadesuit$) (Ours)} \\
\midrule

Medium-Expert & HalfCheetah &
86.7 &
79.8 &
96.1 &
93.5 &
$90.2\pm0.2$ &
$93.3\pm0.3$ &
${\bf 93.9}\pm0.3$ &
\cellcolor{gray!30}
${\bf 93.9}\pm0.5$ \\

Medium-Expert & Hopper &
91.5 &
107.2 &
110.7 &
108.0 &
$100.2\pm3.5$ &
$105.4\pm5.1$ &
$104.4\pm3.1$ &
\cellcolor{gray!30}
${\bf 107.7}\pm5.6$ \\

Medium-Expert & Walker2d &
109.6 &
108.4 &
109.7 &
110.7 &
$108.1\pm0.1$ &
$111.4\pm0.1$ &
$111.4\pm0.1$ &
\cellcolor{gray!30}
${\bf 113.2}\pm0.7$ \\

\midrule
Medium & HalfCheetah &
47.4 &
44.2 &
50.6 &
54.1 &
$53.1\pm0.1$ &
$53.4\pm0.1$ &
$54.8\pm0.1$ &
\cellcolor{gray!30}
${\bf 55.3}\pm0.3$ \\

Medium & Hopper &
66.3 &
58.5 &
82.4 &
98.0 &
$96.2\pm0.5$ &
$71.3\pm2.7$ &
$99.5\pm1.7$ &
\cellcolor{gray!30}
${\bf 101.0}\pm0.3$ \\

Medium & Walker2d &
78.3 &
79.7 &
85.1 &
86.0 &
$83.2\pm1.4$ &
$83.9\pm0.9$ &
$86.5\pm0.2$ &
\cellcolor{gray!30}
${\bf 86.9}\pm0.3$ \\

\midrule
Medium-Replay & HalfCheetah &
44.2 &
42.2 &
47.5 &
47.6 &
$45.0\pm0.3$ &
$42.2\pm0.1$ &
${\bf 47.8}\pm0.4$ &
\cellcolor{gray!30}
$46.9\pm0.3$ \\

Medium-Replay & Hopper &
94.7 &
96.8 &
100.7 &
96.9 &
$93.1\pm0.1$ &
$96.7\pm3.0$ &
$97.4\pm4.0$ &
\cellcolor{gray!30}
${\bf 101.8}\pm0.2$ \\

Medium-Replay & Walker2d &
73.9 &
61.2 &
94.3 &
84.4 &
$69.8\pm4.0$ &
$63.8\pm2.0$ &
$79.3\pm9.7$ &
\cellcolor{gray!30}
${\bf 82.0}\pm8.9$ \\

\midrule
\textbf{Average (Locomotion)} &  &
76.9 &
75.3 &
86.3 &
86.6 &
82.1 &
80.2 &
86.1 &
\cellcolor{gray!30}
{\bf 87.6} \\
\bottomrule
\end{tabular}}
\end{table*}

The results, summarized in Table~\ref{tab:locomotion_colored_best}, demonstrate the effectiveness of our approach. \algname{} achieves the best performance on 8 out of 9 tasks among inference-time adaptation methods, attaining an average score of \textbf{87.6}. Notably, it surpasses the previous state-of-the-art inference-time method (TTS) and competitive training-based baselines, confirming that our efficient adaptation strategy can unlock superior policy performance without requiring additional training.

\section{Conclusion and Limitations}
In this paper, we introduce \algname{}, a training-free inference-time adaptation method for steering pre-trained diffusion models toward high-reward behaviors under generic, potentially non-differentiable reward functions. Our key perspective is to formulate the problem as transporting
the pre-trained generative distribution to a high-reward target distribution. We realize the measure transporting it via Doob's $h$-transform, which induces an additive, time-dependent correction to the sampling dynamics while keeping the pre-trained score model frozen. We develop a Monte Carlo (MC) approximation of the dynamic Doob correction term and establish an end-to-end high-probability convergence guarantee. Moreover, we demonstrate that \algname{} reliably concentrates samples in high-reward regions and achieves strong performance on offline RL benchmarks with competitive sampling efficiency.

One potential limitation of our work is the effectiveness of \algname{} can be sensitive to the probability of high-reward regions $\cE_{\bar X_0}$. Extremely rare high-reward regions can lead to large variance in approximation and reduced efficiency. We hope future work will address this limitation.

\bibliography{ref}
\bibliographystyle{plainnat}

\newpage
\appendix
\appendix
\onecolumn
\section{Specific Forms of Backward Kernels}
\label{app:backward_kernel}

In this section, we provide the explicit forms of the Gaussian transition kernels $p_\theta(x_{t_{l-1}} | x_{t_l})$ defined in Eq.~\eqref{eq:diffusion_posterior} for the sampling algorithms used in our experiments. We utilize the time discretization grid $0=t_0 < t_1 < \dots < t_L=T$, where the reverse process moves from $t_l$ to $t_{l-1}$.

Throughout this section, we relate the model's noise prediction $\epsilon_\theta(x_t, t)$ to the parameterized score function $s_\theta(x, t)$ via:
\begin{align*}
    \epsilon_\theta(x_t, t) = -\sqrt{1-\alpha_t}\, s_\theta(x_t, t).
\end{align*}

\subsection{DDIM Sampler}
The generalized Denoising Diffusion Implicit Model (DDIM)~\citep{song2020denoising} sampler updates the state based on the noise schedule $\alpha_t$ and a stochasticity hyperparameter $\eta \in [0, 1]$. When $\eta=0$, the process is deterministic; when $\eta=1$, it recovers the original DDPM~\citep{ho2020denoising} sampling process.

Given the noisy sample $x_{t_l}$ at time step $t_l$, we first estimate the clean data $\hat{x}_0$ using the model prediction $\epsilon_\theta(x_{t_l}, t_l)$:
\begin{align*}
    \hat{x}_0(x_{t_l}, t_l) = \frac{x_{t_l} - \sqrt{1-\alpha_{t_l}}\,\epsilon_\theta(x_{t_l}, t_l)}{\sqrt{\alpha_{t_l}}}.
\end{align*}
We define the variance of the transition kernel as:
\begin{align*}
    \sigma_{t_l}^2(\eta) = \eta^2 \left( \frac{1-\alpha_{t_{l-1}}}{1-\alpha_{t_l}} \right) \left( 1 - \frac{\alpha_{t_l}}{\alpha_{t_{l-1}}} \right).
\end{align*}
The update rule for the subsequent sample $x_{t_{l-1}}$ combines the estimated clean data, the direction pointing to $x_{t_l}$, and random noise:
\begin{align*}
    x_{t_{l-1}} = \sqrt{\alpha_{t_{l-1}}} \,\hat{x}_0(x_{t_l}, t_l) + \sqrt{1-\alpha_{t_{l-1}} - \sigma_{t_l}^2(\eta)} \,\epsilon_\theta(x_{t_l}, t_l) + \sigma_{t_l}(\eta) \, Z,
\end{align*}
where $Z \sim \mathcal{N}(0, I)$. Substituting $\hat{x}_0$ back into the update rule yields the transition mean $\mu_\theta$. The kernel is given by:
\begin{align*}
    p_\theta(x_{t_{l-1}} | x_{t_l}) = \mathcal{N}\Bigl( \mu_{t_l}(x_{t_l}, s_\theta), \sigma_{t_l}^2(\eta) I \Bigr),
\end{align*}
where the mean is:
\begin{align*}
    \mu_{t_l}(x_{t_l}, s_\theta) = \sqrt{\frac{\alpha_{t_{l-1}}}{\alpha_{t_l}}} x_{t_l} + \left( \sqrt{1-\alpha_{t_{l-1}} - \sigma_{t_l}^2(\eta)} - \sqrt{1-\alpha_{t_l}}\sqrt{\frac{\alpha_{t_{l-1}}}{\alpha_{t_l}}} \right) \epsilon_\theta(x_{t_l}, t_l).
\end{align*}

\subsection{Euler Ancestral SDE}
For the Euler ancestral sampler~\citep{karras2022elucidating}, we adopt the EDM noise parameterization. We first convert the noise schedule values $\alpha_{t_l}$ and $\alpha_{t_{l-1}}$ into the EDM noise standard deviations:
\begin{align*}
    \tilde{\sigma}_{t_l} = \sqrt{\frac{1-\alpha_{t_l}}{\alpha_{t_l}}}, 
    \qquad
    \tilde{\sigma}_{t_{l-1}} = \sqrt{\frac{1-\alpha_{t_{l-1}}}{\alpha_{t_{l-1}}}}.
\end{align*}
The ancestral SDE step splits the transition from $\tilde{\sigma}_{t_l}$ to $\tilde{\sigma}_{t_{l-1}}$ into a deterministic ``down'' step and a stochastic ``up'' step. The variances for these components are defined as:
\begin{align*}
    \tilde{\sigma}_{t_l,\mathrm{up}}^{\,2}
    &= \frac{\tilde{\sigma}_{t_{l-1}}^2}{\tilde{\sigma}_{t_l}^2}
      \bigl(\tilde{\sigma}_{t_l}^2 - \tilde{\sigma}_{t_{l-1}}^2\bigr), \\
    \tilde{\sigma}_{t_l,\mathrm{down}}^{\,2}
    &= \tilde{\sigma}_{t_{l-1}}^2 - \tilde{\sigma}_{t_l,\mathrm{up}}^{\,2}.
\end{align*}
We define the EDM drift direction $d_{t_l}$ directly via the noise prediction:
\begin{align*}
    d_{t_l} = \epsilon_\theta(x_{t_l},t_l).
\end{align*}
A single Euler ancestral SDE step from $t_l$ to $t_{l-1}$ takes the form:
\begin{align*}
    x_{t_{l-1}} = x_{t_l} + d_{t_l}\bigl(\tilde{\sigma}_{t_l,\mathrm{down}} - \tilde{\sigma}_{t_l}\bigr) + \tilde{\sigma}_{t_l,\mathrm{up}}\, Z, \qquad Z \sim \mathcal{N}(0,I).
\end{align*}
Therefore, the transition kernel matches the Gaussian form in Eq.~\eqref{eq:diffusion_posterior}:
\begin{align*}
    p_\theta(x_{t_{l-1}} | x_{t_l}) = \mathcal{N}\Bigl( \mu_{t_l}(x_{t_l}, s_\theta), \sigma_{t_l}^2 I \Bigr),
\end{align*}
with the mean and variance given by:
\begin{align*}
    \mu_{t_l}(x_{t_l}, s_\theta) &= x_{t_l} + \epsilon_\theta(x_{t_l},t_l)\, \bigl(\tilde{\sigma}_{t_l,\mathrm{down}} - \tilde{\sigma}_{t_l}\bigr), \\
    \sigma_{t_l}^2 &= \tilde{\sigma}_{t_l,\mathrm{up}}^{\,2}.
\end{align*}
\section{Proofs}
\subsection{Proof of Lemma~\ref{lem:bounded_tilting_via_event}}\label{append:proof_lem_bounded_tilting_via_event}
\begin{proof}
According to the definition of the tilted density at time $t$ in \eqref{eq:reweight_h}, taking $t=0$ yields:\begin{align}p^h_{\theta,0}(x) = \frac{h(x, 0) p_{\theta,0}(x)}{\PP(\mathcal{E}_{\bar{X}_0})}. \label{eq:tilted_t0}
\end{align}
We first compute the term $h(x,0)$. By the definition of the $h$-function and the event $\mathcal{E}_{\bar{X}_0}$, we have:
\begin{align*}
h(x, 0) &= \PP(\mathcal{E}_{\bar{X}_0} \mid \bar{X}_0 = x) \\
&= \PP\left(U \leq \frac{q(\bar X_0)}{C_q\cdot p_{\theta,0}(\bar X_0)} \middle| \bar{X}_0 = x\right).
\end{align*}
Since the uniform random variable $U$ is independent of the generated sample $\bar{X}_0$, the conditional probability simplifies to:
\begin{align*}
h(x, 0) = \PP\left(U \leq \frac{q(x)}{C_q\cdot p_{\theta,0}(x)}\right).
\end{align*}
Given that $U \sim \mathrm{Unif}(0,1)$ and the assumption $\|q/p_{\theta,0}\|_\infty \le C_q$ implies $0 \le \frac{q(x)}{C_q p_{\theta,0}(x)} \le 1$, the probability is exactly the value of the threshold,
\begin{align}
h(x, 0) = \frac{q(x)}{C_q\cdot p_{\theta,0}(x)}. \label{eq:h_0_val}
\end{align}

Next, we compute the marginal probability of the event $\PP(\mathcal{E}_{\bar{X}_0})$,
\begin{align}
\PP(\mathcal{E}_{\bar{X}_0}) = \int h(x,0) p_{\theta,0}(x) =\int \frac{q(x)}{C_q\cdot p_{\theta,0}(x)} p_{\theta,0}(x) \mathrm{d}x = \frac{1}{C_q}.\label{eq:prob_event}
\end{align}
Finally, substituting \eqref{eq:h_0_val} and \eqref{eq:prob_event} back into \eqref{eq:tilted_t0}, we obtain,
\begin{align*}
p^h_{\theta,0}(x) &= \frac{\frac{q(x)}{C_q\cdot p_{\theta,0}(x)} \cdot p_{\theta,0}(x)}{1/C_q}= q(x).
\end{align*}
This completes the proof.
\end{proof}

\subsection{Proof of Lemma~\ref{lem:score_h_gradient}}
\label{app:proof_lem_score_h_gradient}
\begin{proof}
Recalling the definition of the $h$-function in Definition~\ref{def:h-function}, we can write $\nabla_{x} h(x,t_l)$ as,
\begin{align}\label{eq:log_h_simp}
    \nabla_{x} h(x,t_l)
    &= \nabla_{x} \pathmeasure(\mathcal{E}_{\bar X_0} | \bar X_{t_l}=x)\notag\\
    &= \nabla_{x} \int \pathmeasure(\mathcal{E}_{\bar X_0} | \bar X_0=x_0) \, p_{\theta,0|t_l}(x_0|x) \diff x_0,
\end{align}
where $p_{\theta,0|t_l}(\bar X_0|\bar X_{t_l})$ represents the transition density of the conditional distribution $\bar X_0|\bar X_{t_l}$. In the backward process, the transition density $p_{\theta,0|t_l}(x_0|x_{t_l})$ is induced by \eqref{eq:diffusion_posterior}. It can be represented as:
\begin{align}
    p_{\theta,0|t_l}(x_0|x_{t_l})=\int \prod_{j=1}^l \phi_\theta(x_{t_{j-1}}| x_{t_j}) \diff x_{t_1} \dots \diff x_{t_{l-1}},\label{eq:decompose_p_0_l-1}
\end{align}
where $\phi_{\theta}$ is the Gaussian density defined in \eqref{eq:diffusion_posterior}.
Plugging \eqref{eq:decompose_p_0_l-1} into \eqref{eq:log_h_simp}, we get,
\begin{align}
\label{eq:sim_numerator}
    \;&\nabla_{x} \int \pathmeasure(\mathcal{E}_{\bar X_0} | \bar X_0=x_0) \, p_{\theta,0|t_l}(x_0|x) \diff x_0 \notag\\
    =\;&\nabla_{x}\int
   \pathmeasure(\mathcal{E}_{\bar X_0}  | \bar X_0=x_0)
    \left(\int \prod_{j=1}^l \phi_\theta(x_{t_{j-1}}| x_{t_j}) \diff x_{t_1} \dots \diff x_{t_{l-1}}  \right) \diff x_0 \notag\\
    = \;&\int\!\!\int
   \pathmeasure(\mathcal{E}_{\bar X_0}  | \bar X_0=x_0) \,
    \nabla_{x} \phi_{\theta}(x_{t_{l-1}}|x) \,
    p_{\theta,0|t_{l-1}}(x_0|x_{t_{l-1}})
    \diff x_{t_{l-1}} \diff x_0 \notag\\
    = \;&\int\!\!\int
   h(x_0,0) \,
    \nabla_{x} \log\phi_{\theta}(x_{t_{l-1}}|x) \,
    \phi_{\theta}(x_{t_{l-1}}|x)p_{\theta,0|t_{l-1}}(x_0|x_{t_{l-1}})
    \diff x_{t_{l-1}} \diff x_0 \notag\\
    =\;& \EE\left[
h(\bar X_0,0)
\nabla_{\bar X_{t_l}}\log \phi_{\theta}(\bar X_{t_{l-1}}| \bar X_{t_l})
\Big|\bar X_{t_l}=x
\right],
\end{align}
where the expectation is taken over conditional distribution of the backward trajectory
$(\bar X_{t_{l-1}},\dots,\bar X_0)| \bar X_{t_l}=x$.

The second equality holds due to $\int h(x_0,0)p_{\theta,0|t_{l-1}} (x_0|x_{t_{l-1}})\diff x_0\leq 1$. Furthermore, the Gaussian transition density $\phi_\theta(\cdot|x)$ is continuously differentiable with respect to $x$, and its gradient is dominated by an integrable function due to the exponential decay of the Gaussian tail. Then the interchange of differentiation and integration is justified by the Leibniz integral rule.

Equation~\eqref{eq:sim_numerator} yields the desired expression, which concludes the proof.
\end{proof}
\subsection{Proof of Lemma~\ref{thm:mc_estimation}}\label{app:proof_mc_estimation}
\begin{proof}
In the beginning of the proof, it's helpful to discuss about the Monte Carlo (MC) randomness in the backward sampler. It actually comes from the Gaussian noises used in each reverse transition.

For each trajectory $m\in\{1,\dots,M\}$, let $\{z^{(m)}_j\}_{j=1}^l$ be i.i.d. with
$z^{(m)}_j\sim\cN(0,I_d)$, independent across $j$ and $m$. By the transition kernel in \eqref{eq:diffusion_posterior}
\begin{align}
\bar X_{t_{l-1}}| \bar X_{t_l} = x_{t_l} \sim \mathcal{N}\left(\mu_{t_l}(x_{t_l},s_\theta),\sigma^2_{t_l} I \right),
\end{align}
By the reparameterization trick, we can couple one step as
\[
x^{(m)}_{t_{j-1}}=\mu_{t_j}\left(x^{(m)}_{t_j},s_{\theta}\right)+\sigma_{t_j}\,z^{(m)}_j,
\qquad j=l,l-1,\dots,1,\quad x^{(m)}_{t_l}=x_{t_l}.
\]
Hence the whole backward path is a deterministic function of $(x_{t_l}, z^{(m)}_{1:l})$,
\[
(x^{(m)}_{t_0},\dots,x^{(m)}_{t_{l-1}})=\cG_{\theta,l}\left(x_{t_l}; z^{(m)}_{1:l}\right),
\]
so two MC estimators $\hat h$ and $\nabla \hat h$ can be written as functions of $\cG_{\theta,l}\left(x_{t_l}; z^{(m)}_{1:l}\right)$,
making the MC randomness explicit through $\{z^{(m)}_{1:l}\}_{m=1}^M$ (conditional on $x_{t_l}$).

For simplicity, denote by
\[
z_{\mathrm{mc}}=\{z^{(m)}_j\}_{m=1,\dots,M;\; j=1,\dots,l},\qquad z^{(m)}_j\stackrel{\text{i.i.d.}}{\sim}\cN(0,I_d),
\]
the Gaussian noises used to generate the $M$ Monte Carlo backward trajectories at time $t_l$
(conditional on a given input $x_{t_l}$).

We now begin the proof. To bound the discrepancy between 
$\nabla \log \hat{h}(x_{t_l},t_l)$ and $\nabla \log h(x_{t_l},t_l)$, we write
\begin{align}
    &\;\norm{\nabla \log \hat{h}(x_{t_l},t_l)-\nabla \log h(x_{t_l},t_l)}_2\\
    \leq &\; \left\|\frac{\nabla \hat h(x_{t_l},t_l)-\nabla h(x_{t_l},t_l)}{\hat h(x_{t_l},t_l) \vee \eta_{t_l}}
\right\|_2+\left\|
\nabla h(x_{t_l},t_l)\,\frac{h_{\theta}(x_{t_l},t_l)-(\hat h(x_{t_l},t_l) \vee \eta_{t_l})}{\left(\hat h(x_{t_l},t_l) \vee \eta_{t_l}\right)\,h(x_{t_l},t_l)}\right\|_2 \notag \\
\leq&\; \frac{\left\|\nabla \hat h(x_{t_l},t_l)-\nabla h(x_{t_l},t_l)\right\|_2}{\eta_{t_l}}+\norm{\nabla \log h(x_{t_l},t_l)}_2\frac{\left|h(x_{t_l},t_l)-(\hat h(x_{t_l},t_l) \vee \eta_{t_l})\right|}{\eta_{t_l}}. 
\label{eq:concentrate_h}
\end{align}

By $(a+b)^2\le 2a^2+2b^2$, \eqref{eq:concentrate_h} implies
\begin{align}
    &\;\EE_{X_{t_l}\sim P^h_{\theta,t_l}}\left[ \norm{\nabla \log \hat{h}(X_{t_l},t_l)-\nabla \log h(X_{t_l},t_l)}_2^2 \right]\notag\\
    \leq &\; \underbrace{\EE_{X_{t_l}\sim P^h_{\theta,t_l}} \left[\frac{2\left\|\nabla \hat h(X_{t_l},t_l)-\nabla h(X_{t_l},t_l)\right\|^2_2}{\eta_{t_l}^2}\right]}_{\cH_1}\notag\\
    +&\;\underbrace{\EE_{X_{t_l}\sim P^h_{\theta,t_l}} \left[2\norm{\nabla \log h(X_{t_l},t_l)}_2\frac{\left|h(X_{t_l},t_l)-(\hat h(X_{t_l},t_l) \vee \eta_{t_l})\right|^2}{\eta_{t_l}^2} \right]}_{\cH_{1,2}}.\label{eq:h_1_decomposition}
\end{align}
For the second term $\cH_{1,2}$, we have
\begin{align}
    &\;\EE_{X_{t_l}\sim P^h_{\theta,t_l}} \left[2\norm{\nabla \log h(X_{t_l},t_l)}_2\frac{\left|h(X_{t_l},t_l)-(\hat h(X_{t_l},t_l) \vee \eta_{t_l})\right|^2}{\eta_{t_l}^2} \right]\notag\\
    \leq &\;\EE_{X_{t_l}\sim P^h_{\theta,t_l}} \left[4\left\|\nabla_{X_{t_l}}\log h(X_{t_l},t_l)\right\|_2^2\frac{((h(X_{t_l},t_l)\vee \eta_{t_l})-(\hat h(X_{t_l},t_l)\vee \eta_{t_l}))^2}{\eta_{t_l}^2}\right] \notag\\
    +&\;\EE_{X_{t_l}\sim P^h_{\theta,t_l}} \left[4\left\|\nabla_{X_{t_l}}\log h(X_{t_l},t_l)\right\|_2^2\frac{(h(X_{t_l},t_l)- (h(X_{t_l},t_l)\vee \eta_{t_l}))^2}{\eta_{t_l}^2}\right]\notag\\
    \leq &\;\underbrace{\EE_{X_{t_l}\sim P^h_{\theta,t_l}} \left[4\left\|\nabla_{X_{t_l}}\log h(X_{t_l},t_l)\right\|_2^2\frac{(h(X_{t_l},t_l)-\hat h(X_{t_l},t_l))^2}{\eta_{t_l}^2}\right]}_{\bar{\cH}_2} \notag \\
    +&\;\underbrace{\EE_{X_{t_l}\sim P^h_{\theta,t_l}} \left[4\left\|\nabla_{X_{t_l}}\log h(X_{t_l},t_l)\right\|_2^2\frac{(h(X_{t_l},t_l)- (h(X_{t_l},t_l)\vee \eta_{t_l}))^2}{\eta_{t_l}^2}\right]}_{\cH_3}.\label{eq:h_2_decomposition}
\end{align}
The second inequality holds because $f(x)=x\vee \eta_{t_l}=\max(x,\eta_{t_l})$ is 1-Lipschitz.
We derive an upper bound for $\bar{\cH}_2$ using the following lemma.
\begin{lemma}\label{lem:log_h_nabla_h}
Fix a discretization index $l\in\{1,\dots,L\}$. The following inequality holds,
\begin{align*}
    \|\nabla_{x_{t_l}}\log h(x_{t_l},t_l)\|_2\leq \frac{\|\nabla_{x_{t_l}}\mu_{t_l}(x_{t_l},s_{\theta})\|_2}{\sigma_{t_l}}\sqrt{2\log \frac{1}{h(x_{t_l},t_l)}}.
\end{align*}
\end{lemma}
The proof of Lemma~\ref{lem:log_h_nabla_h} is deferred to Appendix~\ref{app:proof_log_h_nabla_h}.

Then we are ready to bound  $\bar{\cH}_2$
\begin{align*}
\bar{\cH}_2&=\int\left[4\left\|\nabla_{x_{t_l}}\log h(x_{t_l},t_l)\right\|_2^2\frac{(h(x_{t_l},t_l)-\hat h(x_{t_l},t_l))^2}{\eta_{t_l}^2}\right] p_{\theta,t_l}^{h}(x_{t_l}) \diff x_{t_l} .
\end{align*}

By \eqref{eq:h_concentration}, Lemma~\ref{lem:log_h_nabla_h} and Assumption~\ref{assump:bound_h_0} we have
\begin{align}
    \bar{\cH}_2&\leq \int\frac{8\sup_{x_{t_l}}\|\nabla_{x_{t_l}}\mu_{{t_l}}(x_{t_l},s_{\theta})\|_2^2}{\sigma_{t_l}^2\eta_{t_l}^2}(h(x_{t_l},{t_l})-\hat h(x_{t_l},{t_l}))^2\log \frac1{h(x_{t_l},{t_l})} \frac{p_{\theta,{t_l}}(x_{t_l})h(x_{t_l},{t_l})}{\PP(\cE_{\bar X_0})}\diff x_{t_l}\notag\\
    &\leq \EE_{X_{t_l}\sim P_{\theta,{t_l}}}\left[\frac{8C_G^2(h(X_{t_l},{t_l})-\hat h(X_{t_l},{t_l}))^2}{e\sigma_{t_l}^2\eta_{t_l}^2\rho}\right] \notag\\
    &\leq \underbrace{\EE_{X_{t_l}\sim P_{\theta,{t_l}}}\left[\frac{4C_G^2(h(X_{t_l},{t_l})-\hat h(X_{t_l},{t_l}))^2}{\sigma_{t_l}^2\eta_{t_l}^2\rho}\right] }_{\cH_2},
\end{align}
where we invoke $\sup_{x_{t_l}}\|\nabla_{x_{t_l}}\mu_{{t_l}}(x_{t_l},s_{\theta})\|_2\leq C_{G}$. Note that $C_{G}$ is a constant determined solely by the scheduler and $G$, which follows from the fact that $\mu_{{t_l}}$ is a linear combination of $x_{t_l}$ and $s_{\theta}$ with uniformly bounded coefficients, and $\sup_{x}\left\|\nabla_x s_{\theta}(x,t_l)\right\|_2 \le G$ by Assumption~\ref{assump:bound_h_0}.

The second inequality holds due to $\log \frac{1}{h(x_t,t)}h(x_t,t)\leq \frac{1}{e}$ and $h(x_t,t)\in (0,1]$.
Then we can conclude
\begin{align}
    &\EE_{X_t\sim P^h_{\theta,t_l}}\left[ \norm{\nabla \log \hat{h}(X_{t_l},{t_l})-\nabla \log h(X_{t_l},t_l)}_2^2 \right]\leq \cH_1+\cH_2+\cH_3.\label{eq:approx_h_h1h2h3}
\end{align}
Next, we bound $\EE_{z_{\mathrm{mc}}}[\cH_1]$ and $\EE_{z_{\mathrm{mc}}}[\cH_2]$, and then obtain
high-probability bounds for $|\cH_1-\EE_{z_{\mathrm{mc}}}\cH_1|$ and
$|\cH_2-\EE_{z_{\mathrm{mc}}}\cH_2|$ via concentration over the independent noises $z_{\mathrm{mc}}$. In the end, we bound $\cH_3$. 
Combining them together, we can derive a high probability bound for the approximation for $\nabla \log h$ in \eqref{eq:approx_h_h1h2h3}.
\paragraph{Bounding $\EE_{z_{\mathrm{mc}}}[\cH_1]$}

For $\nabla \hat{h}(x_{t_l},{t_l})$, we define $H_i=h(x_0^{(i)},0) \frac{1}{\sigma_{t_l}^2}\nabla_{x_{t_l}}(\mu_{{t_l}}(x_{t_l},s_{\theta}))(x_{t_{l-1}}^{(i)}-\mu_{{t_l}}(x_{t_l},s_{\theta}))$.

By definition, we have $\nabla \hat{h}(x_t,t)=\sum_{i=1}^{M} \frac{1}{M}H_i$, then by Lemma~\ref{lem:score_h_gradient}, we immediately have 
\begin{align*}
    \EE_{z_{\mathrm{mc}}}[\nabla \hat{h}(x_{t_l},t_l)]&=\EE_{( x_0,...,x_{t_{l-1}})|x_{t_l}}[\nabla \hat{h}(x_{t_l},t_l)]\\
    &=\nabla h(x_{t_l},t_l).
\end{align*}
We want to prove it is a sub-Gaussian random vector for a fixed $x_{t_l}$.

Firstly note that $x_{t_{l-1}}^{(i)} \sim \mathcal{N}(\mu_{{t_l}}(x_{t_l},s_{\theta}),\sigma_{t_l}^2 I)$, then we have
\begin{align*}
   \frac{1}{\sigma_{t_l}^2}\nabla_{x_{t_l}}(\mu_{{t_l}}(x_{t_l},s_{\theta}))(x_{t_{l-1}}^{(i)}-\mu_{{t_l}}(x_{t_l},s_{\theta}))\sim \mathcal{N}(0,\sigma_{t_l}^{-2} (\nabla_{x_{t_l}}(\mu_{{t_l}}(x_{t_l},s_{\theta}))(\nabla_{x_{t_l}}(\mu_{{t_l}}(x_{t_l},s_{\theta}))^{\top}))). 
\end{align*}
Then for $\forall v\in \RR^d$, $\norm{v}_2=1$, we have 
\begin{align*}
    v^{\top}(\sigma_{t_l}^{-2} \nabla_{x_{t_l}}(\mu_{{t_l}}(x_{t_l},s_{\theta}))(x_{t_{l-1}}^{(i)}-\mu_{{t_l}}(x_{t_l},s_{\theta}))\sim \mathcal{N}(0,\sigma_{t_l}^{-2} \norm{(\nabla_{x_{t_l}}(\mu_{{t_l}}(x_{t_l},s_{\theta}))v}_2^2 ),
\end{align*}
which implies
\begin{align}
    &\;P(|v^{\top}(\sigma_{t_l}^{-2} \nabla_{x_{t_l}}(\mu_{{t_l}}(x_{t_l},s_{\theta}))(x_{t_{l-1}}^{(i)}-\mu_{{t_l}}(x_{t_l},s_{\theta}))|\geq t)\notag\\
    \leq &\;2\exp\left(-\frac{t^2}{2\sigma_{t_l}^{-2} \norm{(\nabla_{x_{t_l}}(\mu_{{t_l}}(x_{t_l},s_{\theta})))v}_2^2 }\right) \notag\\
   \leq &\; 2\exp\left(-\frac{t^2}{2 \sigma_{t_l}^{-2}\left\|\nabla_{x_{t_l}}(\mu_{{t_l}}(x_{t_l},s_{\theta}))\right\|_2^2 }\right).\label{eq:gaussian_prob}
\end{align}
Notice that $|v^{\top}H_i|\leq |v^{\top}(\sigma_{t_l}^{-2} \nabla_{x_{t_l}}(\mu_{{t_l}}(x_{t_l},s_{\theta}))(x_{t_{l-1}}^{(i)}-\mu_{{t_l}}(x_{t_l},s_{\theta}))|$, by \eqref{eq:gaussian_prob}
\begin{align*}
     P(|v^{\top}H_i|\geq t)
     &\leq P(|v^{\top}(\sigma_{t_l}^{-2} \nabla_{x_{t_l}}(\mu_{{t_l}}(x_{t_l},s_{\theta}))(x_{t_{l-1}}^{(i)}-\mu_{{t_l}}(x_{t_l},s_{\theta}))|\geq t)\\
    &\leq 2\exp\left(-\frac{t^2}{2 \sigma_{t_l}^{-2}\left\|\nabla_{x_{t_l}}(\mu_{{t_l}}(x_{t_l},s_{\theta}))\right\|_2^2 }\right).
\end{align*}
By Lemma~\ref{lem:log_h_nabla_h}, we have
\begin{align}
    \left|v^{\top}\nabla h(x_{t_l},{t_l})\right|\leq
    &\;\norm{\nabla h(x_{t_l},{t_l})}_2\notag\\
    \leq &\;h_{\theta}(x_{t_l},{t_l}) \norm{\nabla \log  h(x_{t_l},{t_l})}_2\notag\\
    \leq  &\;\frac{\|\nabla_{x_{t_l}}(\mu_{{t_l}}(x_{t_l},s_{\theta}))\|_2}{\sigma_{t_l}}\sqrt{2(h(x_{t_l},{t_l}))^2\log \frac1{h(x_{t_l},{t_l})}}\notag\\
    \leq &\;\frac{\|\nabla_{x_{t_l}}(\mu_{{t_l}}(x_{t_l},s_{\theta}))\|_2}{\sigma_{t_l}}\label{eq:nabla_h_bound}.
\end{align}
It also directly implies $\|\nabla h(x_t,t)\|_2\leq \sigma_{t_l}^{-1}\|\nabla_{x_{t_l}}(\mu_{{t_l}}(x_{t_l},s_{\theta}))\|_2\leq \sigma_{t_l}^{-1}C_G$.

Then 
\begin{align}
     &\;P(|v^{\top}H_i-v^{\top}\nabla h(x_{t_l},{t_l})|\geq t)\notag\\
     \leq&\; P(|v^{\top}H_i|+|v^{\top}\nabla h(x_{t_l},{t_l})|\geq t) \notag\\
     \leq & \;P\left(|v^{\top}H_i|\geq t-\sigma_{t_l}^{-1}\|\nabla_{x_{t_l}}(\mu_{{t_l}}(x_{t_l},s_{\theta}))\|_2\right)\notag\\
     \leq& \;2\exp\left(-\frac{\max\left(t-\sigma_{t_l}^{-1}\|\nabla_{x_{t_l}}(\mu_{{t_l}}(x_{t_l},s_{\theta}))\|_2,0\right)^2}{2 \sigma_{t_l}^{-2}\|\nabla_{x_{t_l}}(\mu_{{t_l}}(x_{t_l},s_{\theta}))\|_2^2 }\right). \label{sub_Gaussian_init}
\end{align}
We claim \eqref{sub_Gaussian_init} can imply 
\begin{align*}
    P(|v^{\top}H_i-v^{\top}\nabla h(x_{t_l},{t_l})|\geq t) \leq 2\exp\left(-\frac{t^2}{8 \sigma_{t_l}^{-2}\|\nabla_{x_{t_l}}(\mu_{{t_l}}(x_{t_l},s_{\theta}))\|_2^2 }\right).
\end{align*}
To check it, first consider $t\leq 2\sigma_{t_l}^{-1}\|\nabla_{x_{t_l}}(\mu_{{t_l}}(x_{t_l},s_{\theta}))\|_2$, and we have
\begin{align*}
    2\exp\left(-\frac{\max\left(t-\sigma_{t_l}^{-1}\|\nabla_{x_{t_l}}(\mu_{{t_l}}(x_{t_l},s_{\theta}))\|_2,0\right)^2}{2 \sigma_{t_l}^{-2}\|\nabla_{x_{t_l}}(\mu_{{t_l}}(x_{t_l},s_{\theta}))\|_2^2 }\right)\geq2e^{-1/2}\geq 1,
\end{align*}
and
\begin{align*}
    2\exp\left(-\frac{t^2}{8 \sigma_{t_l}^{-2}\|\nabla_{x_{t_l}}(\mu_{{t_l}}(x_{t_l},s_{\theta}))\|_2^2}\right)\geq2e^{-1/2}\geq 1,
\end{align*}
so these two inequalities both are trivial.

When $t\geq 2\sigma_{t_l}^{-1}\|\nabla_{x_{t_l}}(\mu_{{t_l}}(x_{t_l},s_{\theta}))\|_2$,
\begin{align*}
     2\exp\left(-\frac{t^2}{8 \sigma_{t_l}^{-2}\|\nabla_{x_{t_l}}(\mu_{{t_l}}(x_{t_l},s_{\theta}))\|_2^2}\right)\geq2\exp\left(-\frac{\max\left(t-\sigma_{t_l}^{-1}\|\nabla_{x_{t_l}}(\mu_{{t_l}}(x_{t_l},s_{\theta}))\|_2,0\right)^2}{2 \sigma_{t_l}^{-2}\|\nabla_{x_{t_l}}(\mu_{{t_l}}(x_{t_l},s_{\theta}))\|_2^2 }\right),
\end{align*}
because $t-\sigma_{t_l}^{-1}\|\nabla_{x_{t_l}}(\mu_{{t_l}}(x_{t_l},s_{\theta}))\|_2\geq \frac t2$.
Then we can conclude
\begin{align*}
    P(|v^{\top}H_i-v^{\top}\nabla h(x_{t_l},{t_l})|\geq t) \leq 2\exp\left(-\frac{t^2}{8 \sigma_{t_l}^{-2}\|\nabla_{x_{t_l}}(\mu_{{t_l}}(x_{t_l},s_{\theta}))\|_2^2}\right).
\end{align*}

Then it implies for any $k\in\{1,...,d\}$, let $e_k=(0,\dots,0,1,0,\dots,0)^\top\in\RR^d$
with the $1$ in the $k$-th coordinate, it holds that
\begin{align*}
\Var\bigl(e_k^{\top}(\nabla \hat h(x_{t_l},{t_l})-\nabla h(x_{t_l},{t_l}))\bigr)&=\Var\left(\frac1M\sum_{i=1}^M e_k^{\top}(H_i-\nabla h(x_{t_l},{t_l}))\right)\\
&\leq \frac{1}{M}\Var (e_k^{\top}(H_i-\nabla h(x_{t_l},{t_l})))\\
&\leq \frac{4\|\nabla_{x_{t_l}}(\mu_{{t_l}}(x_{t_l},s_{\theta}))\|_2^2}{\sigma_{t_l}^2 M}.
\end{align*}
This leads to the upper bound of $\EE_{z_{\mathrm{mc}}}[\mathcal{H}_1]$
\begin{align*}
    \EE_{z_{\mathrm{mc}}}[\mathcal{H}_1]&=\int \EE_{z_{\mathrm{mc}}}\left[\frac{2\left\|\nabla \hat h(x_{t_l},{t_l})-\nabla h(x_{t_l},{t_l})\right\|_2^2}{\eta_{t_l}^2}\right] p_{\theta,{t_l}}^h(x_{t_l}) \diff x_{t_l}\\
    &=\int \EE_{z_{\mathrm{mc}}}\left[\frac{2\sum_{k=1}^d\left(e_k^{\top}(\nabla \hat h(x_{t_l},{t_l})-\nabla h(x_{t_l},{t_l}))\right)^2}{\eta_{t_l}^2}\right] p_{\theta,{t_l}}^h(x_{t_l}) \diff x_{t_l}\\
    &\leq \int \frac{8d \sup\|\nabla_{x_{t_l}}(\mu_{{t_l}}(x_{t_l},s_{\theta}))\|_2^2}{\eta_{t_l}^2\sigma_{t_l}^2M} p_{\theta,{t_l}}^h(x_{t_l}) \diff x_{t_l}\\
    &\leq \frac{8d C_G^2}{\eta_{t_l}^2\sigma_{t_l}^2M}.
\end{align*}
The last inequality holds due to Assumption~\ref{assump:bound_h_0} and $\sup_{x_{t_l}}\|\nabla_{x_{t_l}}\mu_{{t_l}}(x_{t_l},s_{\theta})\|_2\leq C_{G}$.
\paragraph{Bounding $\EE_{z_{\mathrm{mc}}}[\mathcal{H}_2]$}
We first prove the concentration bound for $\hat h$. Since $\hat{h}(x_{t_l},{t_l})\in[0,1]$,
apply hoeffding's inequality for bounded random variables yields, for any $\varepsilon>0$,
\[
\mathbb{P}\big(|\hat h(x_{t_l},{t_l}) - h(x_{t_l},{t_l})| \ge \varepsilon\big)
\;\le\;
2\exp\big(-2M\varepsilon^2\big).
\]
It implies
\begin{align}
    \EE_{{z_{\mathrm{mc}}}}[\hat h(x_{t_l},{t_l}) - h(x_{t_l},{t_l})]^2\leq \frac{1}{4M}\label{eq:h_concentration}.
\end{align}
Then we are ready to bound  $\EE_{z_{\mathrm{mc}}}[\mathcal{H}_2]$
\begin{align*}
    \EE_{z_{\mathrm{mc}}}[\mathcal{H}_2]&=\EE_{z_{\mathrm{mc}}}\left[\EE_{X_{t_l}\sim P_{\theta,{t_l}}}\left[\frac{4C_G^2(h(X_{t_l},{t_l})-\hat h(X_{t_l},{t_l}))^2}{\sigma_{t_l}^2\eta_{t_l}^2\rho}\right]\right]  .
\end{align*}
By \eqref{eq:h_concentration}, we have
\begin{align}
    \EE_{z_{\mathrm{mc}}}[\mathcal{H}_2]&\leq \frac{C_{G}^2}{M\sigma_{t_l}^2\eta_{t_l}^2\rho}.
\end{align}

\paragraph{Bounding $|\cH_1-\EE_{z_{\mathrm{mc}}}[\cH_1]|$}
We prove a high probability bound for $|\cH_1-\EE_{z_{\mathrm{mc}}}[\cH_1]|$.
\begin{lemma}[High-probability bound]
\label{lem:H1_conc_indicator}
For any $\delta\in(0,1)$, letting
\[
s=\log\frac{2M}{\delta},
\qquad
R^2=d+2\sqrt{ds}+2s,
\]
we have $\PP(\Omega_R^c)\le \delta/2$ for the event
\[
\Omega_R=\Big\{\max_{1\le i\le M}\|z_l^{(i)}\|_2\le R\Big\}.
\]
Moreover, with probability at least $1-\delta$,
\begin{align*}
    \big|\EE_{z_{\mathrm{mc}}}[\cH_1]-\cH_1\big|&\leq \frac{\sqrt{288}\, C_G^2R^2}{\eta_{t_l}^2\sigma_{t_l}^2}\sqrt{\frac{\log(4/\delta)}{M}}+\frac{24d C_G^2}{\eta_{t_l}^2\sigma_{t_l}^2M}.
\end{align*}
\end{lemma}
\begin{proof}
Using $x_{t_{l-1}}^{(i)}=\mu_{t_l}(x_{t_{l}},s_\theta)+\sigma_{t_l} z_l^{(i)}$, we have for each $i$
\[
H_i=
h(x_0^{(i)},0)\;\frac{1}{\sigma_{t_l}}\nabla_{x_{t_l}}\mu_{t_l}(x_{t_{l}},s_\theta)\,z_l^{(i)}.
\]
Define $\Delta(x_{t_l})=\nabla \hat{h}(x_{t_l},{t_l})-\nabla h(x_{t_l},{t_l})$.

Let
\[
F=\EE_{X_{t_l}\sim P^h_{\theta,{t_l}}}\big[\|\Delta(X_t)\|_2^2\big],
\qquad
\cH_1=2\eta_{t_l}^{-2}F.
\]

Let $z\sim\cN(0,I_d)$. The standard $\chi^2$ tail bound yields for any $s>0$:
\[
\PP\Big(\|z\|_2^2\ge d+2\sqrt{ds}+2s\Big)\le e^{-s}.
\]
With $s=\log\frac{2M}{\delta}$ and $R^2=d+2\sqrt{ds}+2s$, a union bound gives
\[
\PP(\Omega_R^c)=\PP\Big(\max_{i\le M}\|z_l^{(i)}\|_2>R\Big)
\le
\sum_{i=1}^M \PP(\|z_l^{(i)}\|_2>R)
\le
M e^{-s}
=
\frac{\delta}{2}.
\]
On $\Omega_R$, since $\|\nabla_{x_{t_l}}\mu_{t_l}(x_{t_{l}},s_\theta)\|_2\le C_G$,
\[
\|H_i\|_2 \le \frac{1}{\sigma_{t_l}}\|\nabla_{x_{t_l}}\mu_{t_l}(x_{t_{l}},s_\theta)\|_2\|z_l^{(i)}\|_2
\le \frac{C_{G}R}{\sigma_{t_l}},
\qquad\forall\,i.
\]
Define $B=\frac{C_{G}R}{\sigma_{t_l}}$, consequently,
\[
\Big\|\frac1M\sum_{j=1}^M H_j\Big\|_2 \le B,
\qquad
 \|\nabla h(x_{t_l},t_l)\| \leq \frac{C_{G}}{\sigma_{t_l}}\leq \frac{C_{G}}{\sigma_{t_l}}R=B,
\qquad\Rightarrow\qquad
\|\Delta(x_{t_l})\|_2\le 2B,
\]
all on $\Omega_R$.

Now on $\Omega_R$, replace only the $i$-th MC randomness block $\{Z_{j}^{(i)}\}_{j=1}^l$ by an independent copy, yielding $H_i'$, $F'$ and
\[
\Delta'(x_{t_l})=\Delta(x_{t_l})+\delta(x_{t_l}),
\qquad
\delta(x_{t_l})=\frac{1}{M}\big(H_i'-H_i\big).
\]
On $\Omega_R$ (for both the original and the replaced samples),
\[
\|\delta(x_{t_l})\|_2\le \frac{1}{M}\big(\|H_i'\|_2+\|H_i\|_2\big)\le \frac{2B}{M}.
\]
Using $\big|\|a+b\|_2^2-\|a\|_2^2\big|\le 2\|a\|_2\|b\|_2+\|b\|_2^2$ and $\|\Delta(x_{t_l})\|\le 2B$, we get
\[
\big|\|\Delta'(x_{t_l})\|_2^2-\|\Delta(x_{t_l})\|_2^2\big|
\le 2\|\Delta(x_{t_l})\|_2\|\delta(x_{t_l})\|_2+\|\delta(x_{t_l})\|_2^2
\le 2(2B)\frac{2B}{M}+\Big(\frac{2B}{M}\Big)^2
\le \frac{12B^2}{M}.
\]
Taking $\EE_{X_t}$ preserves the bound, so on $\Omega_R$,
\[
|F-F^{(i)}|\le c_i,
\qquad
c_i=\frac{12B^2}{M}.
\]

Conditional on $\Omega_R=\bigcap_{1\leq i\leq M}\{\|z_l^{(i)}\|_2\le R\}$, $\{\{Z_{i}^{(m)}\}_{i=1}^l\}_{m=1}^M$ are still mutually independent, McDiarmid's inequality gives for any $\varepsilon>0$,
\[
\PP\big(|F-\EE_{z_{\mathrm{mc}}}[F|\Omega_R]|\ge \varepsilon\ \big|\ \Omega_R\big)
\le
2\exp\Big(-\frac{2\varepsilon^2}{\sum_{i=1}^M c_i^2}\Big).
\]
Since $\sum_{i=1}^M c_i^2=M\cdot(12B^2/M)^2=144B^4/M$, we obtain
\[
\PP\big(|F-\EE_{z_{\mathrm{mc}}}[F|\Omega_R]|\ge \varepsilon\ \big|\ \Omega_R\big)
\le
2\exp\Big(-\frac{M\varepsilon^2}{72B^4}\Big).
\]
Set $\delta_1=\delta/2$ and choose
\[
\varepsilon
=
B^2\sqrt{\frac{72\log(2/\delta_1)}{M}}
=
B^2\sqrt{\frac{72\log(4/\delta)}{M}},
\]
so that the RHS is at most $\delta_1=\delta/2$. Multiplying by $2\eta_{t_l}^{-2}$ yields,
\begin{align*}
\PP(|\cH_1-\EE_{z_{\mathrm{mc}}}[\cH_1|\Omega_R]|\geq 2\eta_{t_l}^{-2}\varepsilon)&\leq \PP(|\cH_1-\EE_{z_{\mathrm{mc}}}[\cH_1|\Omega_R]|\geq 2\eta_{t_l}^{-2}\varepsilon|\Omega_R)+\PP(\Omega_R^{c})\\ 
&\leq \frac{\delta}{2}+ \frac{\delta}{2}\\
&=\delta,
\end{align*}
where 
\begin{align*}
2\eta_{t_l}^{-2}\varepsilon
=
\frac{\sqrt{288}\,C_{G}^2R^2}{\eta_{t_l}^2\sigma_{t_l}^2}\sqrt{\frac{\log(4/\delta)}{M}}.  
\end{align*}

Since
\[
\big|\EE_{z_{\mathrm{mc}}}[\cH_1|\Omega_R]-\EE_{z_{\mathrm{mc}}}[\cH_1]\big|
\leq\EE_{z_{\mathrm{mc}}}[\cH_1]+\EE_{z_{\mathrm{mc}}}[\cH_1|\Omega_R]\leq \left(1+\frac{1}{P(\Omega_R)}\right) \EE_{z_{\mathrm{mc}}}[\cH_1]\leq 3\EE_{z_{\mathrm{mc}}}[\cH_1]
\]

On $\Omega_R$, by triangle inequality,
\[
|\cH_1-\EE_{z_{\mathrm{mc}}}[\cH_1]|
\le
|\cH_1-\EE_{z_{\mathrm{mc}}}[\cH_1|\Omega_R]|
+
|\EE_{z_{\mathrm{mc}}}[\cH_1|\Omega_R]-\EE_{z_{\mathrm{mc}}}[\cH_1]|.
\]
Since $\EE_{z_{\mathrm{mc}}}[\cH_1]\leq\frac{d C_{G}^2}{\eta_{t_l}^2\sigma_{t_l}^2M}$, then we can conclude, the following inequality holds with probability at least $1-\delta$
\begin{align}
    \big|\EE_{z_{\mathrm{mc}}}[\cH_1]-\cH_1\big|&\leq\big|\EE_{z_{\mathrm{mc}}}[\cH_1|\Omega_R]-\cH_1\big|+\big|\EE_{z_{\mathrm{mc}}}[\cH_1|\Omega_R]-\EE_{z_{\mathrm{mc}}}[\cH_1]\big|\notag \\
    &\leq \frac{\sqrt{288}\,C_{G}^2R^2}{\eta_{t_l}^2\sigma_{t_l}^2}\sqrt{\frac{\log(4/\delta)}{M}}+\frac{24d C_{G}^2}{\eta_{t_l}^2\sigma_{t_l}^2M}. \label{eq:h_1_highprob}
\end{align}
\end{proof}
\paragraph{Bounding $|\cH_2-\EE_{z_\mathrm{mc}}[\cH_2]|$ with high probability.}
Define the constant
$C_{\cH_2}=\frac{4C_{G}^2}{\sigma_{t_l}^2\eta_{t_l}^2\rho}$,
Then
\begin{align*}
   \cH_2 = C_{\cH_2}\;\EE_{X_{t_l}\sim P_{\theta,t_l}}\Big[(h(X_{t_l},{t_l})-\hat h(X_{t_l},{t_l}))^2\Big]. 
\end{align*}

\begin{lemma}[McDiarmid concentration for $\cH_2$]
\label{lem:H2_conc_avgXt}
For any $\delta\in(0,1)$, with probability at least $1-\delta$ over the MC randomness,
\[
\big|\cH_2-\EE_{z_{\mathrm{mc}}}[\cH_2]\big|
\;\le\;
C_{\cH_2} \frac{3}{\sqrt{2M}}\sqrt{\log\frac{2}{\delta}}
\;=\;
\frac{12C_{G}^2}{\sigma_{t_l}^2\eta_{t_l}^2\rho}\,
\sqrt{\frac{\log(2/\delta)}{2M}}.
\]
\end{lemma}

\begin{proof}
Introduce
\[
Y(x_{t_l})=\hat h(x_{t_l},{t_l})-h(x_{t_l},{t_l}),
\qquad
\bar{F}=\EE_{X_{t_l}\sim P_{\theta,{t_l}}}\big[Y(X_{t_l})^2\big],
\]
so that $\cH_2=C_{\cH_2}\,\bar{F}$.

Fix an index $i\in\{1,\dots,M\}$ and consider replace only the $i$-th MC randomness block $\{Z_{j}^{(i)}\}_{j=1}^l$ by an independent copy, producing 
$\hat h'(x_{t_l},{t_l})$, $Y'(x_{t_l})$ and $\bar F'$. Define
\[
\Delta(x_{t_l})=\hat h'(x_{t_l},{t_l})-\hat h(x_{t_l},{t_l})=\frac{h((x_0^{(i)})',0)-h(x_0^{(i)},0)}{M}.
\]
Since $h((x_0^{(i)})',0),h(x_0^{(i)},0)\in [0,1]$, we have for all $x_{t_l}$,
\[
|\Delta(x_{t_l})|\le \frac{1}{M}.
\]
Moreover, since $\hat h(x_{t_l},{t_l})\in[0,1]$ and $h(x_{t_l},{t_l})\in[0,1]$, it follows that
\[
Y(x_{t_l})=\hat h(x_{t_l},{t_l})-h(x_{t_l},{t_l})\in[-1,1],
\qquad\Rightarrow\qquad |Y(x_{t_l})|\le 1
\quad\text{for all }x_{t_l}.
\]
Let $Y'(x_{t_l})=\hat h'(x_{t_l},{t_l})-h(x_{t_l},{t_l})=Y(x_{t_l})+\Delta(x_{t_l})$. Then we have
\[
Y'(x_{t_l})^2-Y(x_{t_l})^2
=
\big(Y(x_{t_l})+\Delta(x_{t_l})\big)^2-Y(x_{t_l})^2
=
2Y(x_{t_l})\Delta(x_{t_l})+\Delta(x_{t_l})^2.
\]
Therefore,
\[
\big|Y'(x_{t_l})^2-Y(x_{t_l})^2\big|
\le 2|Y(x_{t_l})|\,|\Delta(x_{t_l})|+|\Delta(x_{t_l})|^2
\le 2\cdot 1\cdot \frac{1}{M}+\frac{1}{M^2}
\le \frac{3}{M}.
\]
Taking expectation over $X_{t_l}\sim P_{\theta,t_l}$ preserves the bound, hence
\[
|\bar{F}'-\bar{F}|
=
\left|\EE_{X_{t_l}\sim P_{\theta,t_l}}\big[Y'(X_{t_l})^2-Y(X_{t_l})^2\big]\right|
\le \EE_{X_{t_l}\sim P_{\theta,t_l}}\big|Y'(X_{t_l})^2-Y(X_{t_l})^2\big|
\le \frac{3}{M}.
\]
Thus $\bar{F}$ satisfies bounded differences with constants $\bar{c}_i=3/M$ for all $i$.
By McDiarmid's inequality, for any $\varepsilon>0$,
\[
\PP\Big(|\bar{F}-\EE_{z_{\mathrm{mc}}}[\bar{F}]|\ge \varepsilon\Big)
\le
2\exp\!\left(
-\frac{2\varepsilon^2}{\sum_{i=1}^M \bar{c}_i^2}
\right).
\]
Since $\sum_{i=1}^M \bar{c}_i^2 = M\cdot (3/M)^2 = 9/M$, we obtain
\[
\PP\Big(|\bar{F}-\EE_{z_{\mathrm{mc}}}[\bar{F}]|\ge \varepsilon\Big)
\le
2\exp\!\left(
-\frac{2M}{9}\varepsilon^2
\right).
\]

Because $\cH_2=C_{\cH_2}\,\bar{F}$,
\[
\PP\Big(|\cH_2-\EE[\cH_2]|\ge \tau\Big)
=
\PP\Big(|\bar{F}-\EE[\bar{F}]|\ge \tau/C_{\cH_2}\Big)
\le
2\exp\!\left(
-\frac{2M}{9}\left(\frac{\tau}{C_{\cH_2}}\right)^2
\right),
\]
which proves the tail bound. Setting the right-hand side equal to $\delta$ gives
\begin{align}
  |\cH_2-\EE[\cH_2]|
\le
C_{\cH_2}\sqrt{\frac{9}{2M}\log\frac{2}{\delta}}
=
\frac{12C_{G}^2}{\sigma_{t_l}^2\eta_{t_l}^2\rho}\frac{3}{\sqrt{2M}}\sqrt{\log\frac{2}{\delta}}, \label{eq:H_2_highprob} 
\end{align}
with probability at least $1-\delta$, completing the proof.
\end{proof}

\paragraph{Bounding $\cH_3$}
Set $\eta_{t_l}\leq 1/e$, by Lemma~\ref{lem:log_h_nabla_h} and Assumption~\ref{assump:bound_h_0}, we have
\begin{align}
    \cH_{3}&\leq  \int_{h(x_{t_l},{t_l})\leq \eta_{t_l} } \frac{8\sup_{x_{t_l}}\|\nabla_{{t_l}}\mu_{\theta}(x_{t_l},s_{\theta})\|^2}{\sigma_{t_l}^2}\log \frac1{h(x_{t_l},{t_l})} \frac{p_{\theta,{t_l}}(x_{t_l})h(x_{t_l},{t_l})}{\PP(\cE_{\bar X_0})}\diff x_{t_l}\notag\\
    &\leq \int_{h(x_{t_l},{t_l})\leq \eta_{t_l} } \frac{8C_G^2}{\sigma_{t_l}^2}\log \frac1{\eta_{t_l}} \frac{p_{\theta,{t_l}}(x_{t_l})\eta_{t_l}}{\PP(\cE_{\bar X_0})}\diff x_{t_l}\notag\\
    &\leq \frac{8C_G^2\eta_{t_l}}{\rho\sigma_{t_l}^2}\log \frac1{\eta_{t_l}} \label{eq:result_h21}.
\end{align}
Putting all together,
\begin{align*}
    &\EE_{X_{t_l}\sim P^h_{\theta,{t_l}}}\left[ \norm{\nabla \log \hat{h}(X_{t_l},{t_l})-\nabla \log h(X_{t_l},{t_l})}_2^2 \right]\\
    \leq& \cH_1+\cH_2+\cH_3\\
    \leq& \cH_3+\EE_{z_{\mathrm{mc}}}[\cH_1+\cH_2]+|\cH_1-\EE_{z_{\mathrm{mc}}}[\cH_1] |+|\cH_2-\EE_{z_{\mathrm{mc}}}[\cH_2] |\\
    \leq &\frac{8C_{G}^2\eta_{t_l}}{\rho\sigma_{t_l}^2}\log \frac1{\eta_{t_l}}+\frac{8d C_{G}^2}{\eta_{t_l}^2\sigma_{t_l}^2M}+\frac{C_{G}^2}{M\sigma_{t_l}^2\eta_{t_l}^2\rho}+|\cH_1-\EE_{z_{\mathrm{mc}}}[\cH_1] |+|\cH_2-\EE_{z_{\mathrm{mc}}}[\cH_2]|.
\end{align*}
By high probability bounds for $|\cH_1-\EE_{z_{\mathrm{mc}}}[\cH_1] |$ and $|\cH_2-\EE_{z_{\mathrm{mc}}}[\cH_2]|$ in \eqref{eq:h_1_highprob} and \eqref{eq:H_2_highprob}, the following inequality holds with probability at least $1-2\delta$,
\begin{align*}
    &\EE_{X_{t_l}\sim P^h_{\theta,{t_l}}}\left[ \norm{\nabla \log \hat{h}(X_{t_l},{t_l})-\nabla \log h(X_{t_l},{t_l})}_2^2 \right]\\
    \lesssim &\frac{C_{G}^2\eta_{t_l}}{\rho\sigma_{t_l}^2}\log \frac1{\eta_{t_l}}+\frac{d C_{G}^2}{\eta_{t_l}^2\sigma_{t_l}^2M}+\frac{C_{G}^2}{M\sigma_{t_l}^2\eta_{t_l}^2\rho}+\frac{C_{G}^2}{\sigma_{t_l}^2\eta_{t_l}^2\rho}\,
\sqrt{\frac{\log(2/\delta)}{M}}\\
+&\frac{\,C_{G}^2R^2}{\eta_{t_l}^2\sigma_{t_l}^2}\sqrt{\frac{\log(4/\delta)}{M}}\\
\lesssim& \frac{\eta_{t_l}}{\sigma_{t_l}^2}\log \frac1{\eta_{t_l}}+\frac{ 1}{\eta_{t_l}^2\sigma_{t_l}^2}\left(\frac1M+\sqrt{\frac{\log(2/\delta)}{M}}+\,\log \frac{2M}{\delta}\sqrt{\frac{\log(4/\delta)}{M}}\right)\\
\lesssim&\frac{\eta_{t_l}}{\sigma_{t_l}^2}\log \frac1{\eta_{t_l}}+\frac{ 1}{\eta_{t_l}^2\sigma_{t_l}^2}\left(\,\log \frac{2M}{\delta}\sqrt{\frac{\log(4/\delta)}{M}}\right).
\end{align*}
Here, $M$ is sufficiently large to satisfy $M \geq C_{M}$, where $C_{M}$ depends on $C_{G}$ and other relevant constants.
Choose $\eta_{t_l}=\min\left(\frac{1}{M^{1/6}},\frac{1}e\right)=M^{-1/6}$, then 
\begin{align*}
    &\EE_{X_{t_l}\sim P^h_{\theta,t_l}}\left[ \norm{\nabla \log \hat{h}(X_{t_l},{t_l})-\nabla \log h(X_{t_l},{t_l})}_2^2 \right]\\\lesssim&\frac{1}{\sigma_{t_l}^2M^{1/6}}\log M+\frac{ 1}{M^{1/6}\sigma_{t_l}^2}\left(\,\log \frac{2M}{\delta}\sqrt{\log(4/\delta)}\right)\\
    \lesssim&\frac{ 1}{M^{1/6}\sigma_{t_l}^2}\left(\,\log \frac{M}{\delta}\sqrt{\log(1/\delta)}\right)
\end{align*}
holds with probability at least $1-2\delta$. It also imply
\begin{align*}
    &\EE_{X_{t_l}\sim P^h_{\theta,{t_l}}}\left[ \norm{\nabla \log \hat{h}(X_{t_l},{t_l})-\nabla \log h(X_{t_l},{t_l})}_2^2 \right]\\=&~\cO\left(\frac{1}{\sigma_{t_l}^2}\frac{ \log \frac{M}{\delta}\sqrt{\log(1/\delta)}}{M^{1/6}}\right)
\end{align*}
holds with probability at least $1-\delta $.
\end{proof}
\subsection{Proof of Theorem~\ref{thm:tv_decomposition}}\label{append:tv_decomposition}
\begin{proof}
Since total variation satisfies the triangle inequality, we have 
\begin{align*}
    \mathrm{TV}\!\left(P_{\mathrm{data}}(\cdot|\cE_{X_0}), \hat P_{\theta}^{h}\right)
\;\le\;
\mathrm{TV}\!\left(P_{\mathrm{data}}(\cdot|\cE_{X_0}), P_{\theta}(\cdot|\cE_{\bar X_0})\right)
\;+\;
\mathrm{TV}\!\left(P_{\theta}(\cdot|\cE_{\bar X_0}), \hat P_{\theta}^{h}\right).
\end{align*}
\noindent
The term $\mathrm{TV}\!\left(P_{\mathrm{data}}(\cdot|\cE_{X_0}), P_{\theta}(\cdot|\cE_{\bar X_0})\right)$ measures the modeling error between the ideal conditional data distribution and the conditional generated
distribution, while
$\mathrm{TV}\!\left(P_{\theta}(\cdot|\cE_{\bar X_0}), \hat P_{\theta}^{h}\right)$ measures the sampling error between the conditional generated
distribution
and the output distribution of \algname{}.

\paragraph{Bounding $\mathrm{TV}\!\left(P_{\mathrm{data}}(\cdot|\cE_{X_0}), P_{\theta}(\cdot|\cE_{\bar X_0})\right)$}
Define $h_{\mathrm{data}}$ as the corresponding $h$-function defined by $\cE_{X_0}$.
By definition, $\cE_{X_0}$ are defined analogously by replacing $\bar X_0$ with $X_0$, then $h_{\mathrm{data}}(x,0)=h(x,0)$.
\begin{align*}
\mathrm{TV}\!\left(P_{\mathrm{data}}(\cdot|\cE_{X_0}), P_{\theta}(\cdot|\cE_{\bar X_0})\right)
&= \frac12 \int_{x} \left| \frac{h(x,0)p_{\theta,0}(x)}{\PP(\cE_{\bar X_0})} - \frac{p_{\mathrm{data}}(x)h_{\mathrm{data}}(x,0)}{\PP(\cE_{X_0})} \right|\,\diff x \\
&= \frac{1}{2\,\PP(\cE_{\bar X_0})\PP(\cE_{X_0})}
   \int_{x} h(x,0)\left| p_{\theta,0}(x)\PP(\cE_{X_0})- p_{\mathrm{data}}(x)\PP(\cE_{\bar X_0}) \right|\diff x\\
&\leq \frac{1}{2\,\PP(\cE_{\bar X_0})}
   \int_{x} \left| p_{\theta,0}(x)- p_{\mathrm{data}}(x) \right|\diff x\\
   &+\frac{1}{2\,\PP(\cE_{\bar X_0})\PP(\cE_{X_0})}\int \left| \PP(\cE_{\bar X_0})- \PP(\cE_{X_0})\right|h(x,0)p_{\mathrm{data}}(x)\diff x\\
&\lesssim \frac{1}{\rho}\mathrm{TV}(P_{\mathrm{data}},P_{\theta}).
\end{align*}
Here $p_{\mathrm{data}}(x)$ is the density of $P_{\mathrm{data}}$. 
\paragraph{Bounding $\mathrm{TV}\!\left(P_{\theta}(\cdot|\cE_{\bar X_0}), \hat P_{\theta}^{h}\right)$}
By definition, $P_{\theta}(\cdot|\cE_{\bar X_0})$ is the terminal distribution of \eqref{eq:piecewise_sde_corrected}, 
\begin{align}
\diff \bar{X}_t^h
\!=\!
\left[-\tfrac{1}{2}\bar{X}_{t}^h - s_{\theta}(\bar{X}_{t_l}^h,t_l)-\nabla \log h(\bar{X}_{t}^h,t)\right]\diff t +\diff \overline{W}_t,
\end{align}
and $\bar X_0^h\sim P_{\theta}(\cdot|\cE_{\bar X_0})$.

Meanwhile $\hat{P}_{\theta}^{h}$ is the terminal distribution of the following SDE,
\begin{align}
\diff \bar{X}_t^h
\!=\!
\left[\!-\tfrac{1}{2}\bar{X}_{t}^h - s_{\theta}(\bar{X}_{t_l}^h,t_l)-\nabla \log \hat{h}(\bar{X}_{t_l}^h,t_l)\right]\!\diff t +\diff \overline{W}_t,
\end{align}
where $\bar X_0^h\sim \hat P_{\theta}^h$.

We use Theorem 9 in \citep{chen2022sampling} to bound $\mathrm{TV}\!\left(P_{\theta}(\cdot|\cE_{\bar X_0}), \hat P_{\theta}^{h}\right)$.
Firstly we check 
\begin{align*}
    &\;\sum_{l=1}^L \EE_{\,\PP_{\theta}^h}\int_{t_{l-1}}^{t_l}\left[\|\nabla \log \hat{h}(\bar{X}_{t_l}^h,t_l)-\nabla \log h(\bar{X}_{t}^h,t)\|_2^2\right]\diff t\\
    \leq &\;2\sum_{l=1}^L \EE_{\,\PP_{\theta}^h}\int_{t_{l-1}}^{t_l}\left[\|\nabla \log h(\bar{X}_{t_l}^h,t_l)-\nabla \log h(\bar{X}_{t}^h,t)\|_2^2\right]\diff t\\
    + &\; 2\sum_{l=1}^L \EE_{\,\PP_{\theta}^h}\int_{t_{l-1}}^{t_l}\left[\|\nabla \log h(\bar{X}_{t_l}^h,t_l)-\nabla \log \hat{h}(\bar{X}_{t_l}^h,t_l)\|_2^2\right]\diff t\\
    \lesssim &\; \varepsilon_{\mathrm{dis}}T+\frac{ \log \frac{LM}{\delta}\sqrt{\log(L/\delta)}}{M^{1/6}}\frac{T}{L}\sum_{l=1}^L  \sigma_{t_l}^{-2}\\
    <&\;\infty
\end{align*}
holds with probability at $1-\frac{\delta L}{L}=1-\delta$ by Lemma~\ref{thm:mc_estimation} (For each discretization index $l\in\{1,\dots,L\}$, the MC approximation bound holds with probability at least $1-\delta/L$).

Then using the same argument in Theorem 9 in \citep{chen2022sampling}, we can conclude
\begin{align*}
    \mathrm{KL}(P_{\theta}(\cdot|\cE_{\bar X_0}),\hat{P}_{\theta}^h)
    &\lesssim\varepsilon_{\mathrm{dis}}T+\frac{ \log \frac{LM}{\delta}\sqrt{\log(L/\delta)}}{M^{1/6}}\frac{T}{L}\sum_{l=1}^L  \sigma_{t_l}^{-2}\\
    &\lesssim\varepsilon_{\mathrm{dis}}T+\frac{\kappa_{\sigma} \log \frac{M}{\delta}\sqrt{\log(1/\delta)}}{M^{1/6}},
\end{align*}
where $\frac{T}{L}\sum_{l=1}^L 1/\sigma_{t_l}^2$ denotes as $\kappa_{\sigma}$.
By Pinsker’s Inequality,
\begin{align*}
    \mathrm{TV}(P_{\theta}(\cdot|\cE_{\bar X_0}),\hat{P}_{\theta}^h)&\lesssim\sqrt{\mathrm{KL}(P_{\theta}(\cdot|\cE_{\bar X_0}),\hat{P}_{\theta}^h)}\\
    &\lesssim \sqrt{\varepsilon_{\mathrm{dis}}T+\frac{\kappa_{\sigma} \log \frac{M}{\delta}\sqrt{\log(1/\delta)}}{M^{1/6}}}
\end{align*}
holds with probability at $1-\delta$.

Then we can conclude
\begin{align*}
    \mathrm{TV}\!\left(P_{\mathrm{data}}(\cdot|\cE_{X_0}), \hat P_{\theta}^{h}\right)
\;&\le\;
\mathrm{TV}\!\left(P_{\mathrm{data}}(\cdot|\cE_{X_0}), P_{\theta}(\cdot|\cE_{\bar X_0})\right)
\;+\;
\mathrm{TV}\!\left(P_{\theta}(\cdot|\cE_{\bar X_0}), \hat P_{\theta}^{h}\right)\\
&\lesssim\;\frac{1}{\rho}\,\mathrm{TV}(P_{\mathrm{data}}, P_{\theta})+\sqrt{\varepsilon_{\mathrm{dis}}T+\frac{\kappa_{\sigma} \log \frac{M}{\delta}\sqrt{\log(1/\delta)}}{M^{1/6}}}.
\end{align*}
\end{proof}
\subsection{Proof of Lemma~\ref{lem:log_h_nabla_h}}\label{app:proof_log_h_nabla_h}
\begin{proof}
By definition, $h(x_{t_l},{t_l})=\int h(x_{t_{l-1}},{t_{l-1}})\phi_{\theta}(x_{t_{l-1}}|x_{t_l}) \diff x_{t_{l-1}}$.

We define the reweighed distribution $Q_{\theta}(X_{t_{l-1}}|X_{t_l})$ with density $q_{\theta}(x_{t_{l-1}}|x_{t_l})=\frac{h(x_{t_{l-1}},{t_{l-1}})\phi_{\theta}(x_{t_{l-1}}|x_{t_l})}{h(x_{t_l},{t_l})}$, then we have
\begin{align}
    \nabla_{x_{t_l}}\log h(x_{t_l},t_l)&=\int\frac{1}{\sigma_{t_l}^2}\nabla_{x_{t_l}}(\mu_{t_l}(x_{t_l},s_{\theta}))(x_{t_{l-1}}-\mu_{t_l}(x_{t_l},s_{\theta}))q_{\theta}(x_{t_{l-1}}|x_{t_l}) \diff{x_{t_{l-1}}}  \notag\\
    &=\frac{1}{\sigma_{t_l}^2}\nabla_{x_{t_l}}(\mu_{t_l}(x_{t_l},s_{\theta}))\left[\EE_{X_{t_{l-1}}\sim Q_{\theta}(\cdot|x_{t_l})}[X_{t_{l-1}}]-\mu_{t_l}(x_{t_l},s_{\theta})\right]  \notag \\
    &=\frac{1}{\sigma_{t_l}^2}\nabla_{x_{t_l}}(\mu_{t_l}(x_{t_l},s_{\theta}))\left[\EE_{X_{t_{l-1}}\sim Q_{\theta}(\cdot|x_{t_l})}[X_{t_{l-1}}]-\EE_{X_{t_{l-1}}\sim \Phi(\cdot|x_{t_l})}[X_{t_{l-1}}]\right] \label{ineq:logh_to_W2},
\end{align}
where $\Phi(\cdot|x_{t})$ is $\cN(\mu_{t_l}(x_{t_l},s_{\theta}),\sigma_{t_l}^2 I)$.

It's obvious that $q_{\theta}\ll\phi_{\theta}$. By Talagrand’s transportation inequality \citep{talagrand1996transportation,djellout2004transportation,xie2025constrained}, we have
\begin{align}
    \|\EE_{X_{t_{l-1}}\sim Q_{\theta}(\cdot|x_{t_l})}[X_{t_{l-1}}]-\EE_{X_{t_{l-1}}\sim \Phi(\cdot|x_{t_l})}[X_{t_{l-1}}]\|_2^2 &\leq W_2^2(Q_{\theta}(\cdot|x_{t_l}),\Phi_{\theta}(\cdot|x_{t_l}))\leq 2\sigma_{t_l}^2 \mathrm{KL}(Q_{\theta}(\cdot|x_{t_l}),\Phi_{\theta}(\cdot|x_{t_l})) \label{ineq:W2_to_KL}.
\end{align}
Now we turn to compute the KL divergence between $Q_{\theta}(\cdot|x_{t_l})$ and $\Phi_{\theta}(\cdot|x_{t_l})$
\begin{align}
    \mathrm{KL}(Q_{\theta}(\cdot|x_{t_l})||\Phi_{\theta}(\cdot|x_{t_l}))&=\int q_{\theta}(x_{t_{l-1}}|x_{t_l}) \log \frac{q_{\theta}(x_{t_{l-1}}|x_{t_l})}{\phi_{\theta}(x_{t_{l-1}}|x_{t_l})} \diff x_{t_{l-1}} \notag\\
    &=\int q_{\theta}(x_{t_{l-1}}|x_{t_l}) \log \frac{h(x_{t_{l-1}},t_{l-1})}{h(x_{t_l},t_l)} \diff x_{t_{l-1}} \notag\\
    &=\int q_{\theta}(x_{t_{l-1}}|x_{t_l}) (\log h(x_{t_{l-1}},t_{l-1})-\log h(x_{t_l},t_l)) \diff x_{t_{l-1}} \notag\\
    &\leq -\log h(x_{t_l},t_l).\label{ineq:KL_pq}
\end{align}
By combining \eqref{ineq:logh_to_W2}, \eqref{ineq:W2_to_KL}, \eqref{ineq:KL_pq}, we can conclude
\begin{align*}
    \|\nabla_{x_{t_l}}\log h(x_{t_l},t_l)\|_2\leq \frac{\|\nabla_{x_{t_l}}(\mu_{t_l}(x_{t_l},s_{\theta}))\|_2}{\sigma_{t_l}}\sqrt{2\log \frac1{h(x_{t_l},t_l)}}.
\end{align*}
\end{proof}
\section{Experimental Details and Additional Results}

\subsection{Formal Construction of the Conditioning Event $\mathcal{E}_{\bar X_0}$}
\label{app:event_construction}

In Section~\ref{sec:practical_ver}, we defined the terminal $h$-function as $h(x,0) \propto \exp(r(x)/\tau)$. Here, we rigorously define the event $\mathcal{E}_{\bar X_0}$ that induces this function by invoking Lemma~\ref{lem:bounded_tilting_via_event}.

The target distribution is defined via exponential tilting as $q(x) \propto p_{\theta,0}(x)\exp(r(x)/\tau)$. To satisfy the boundedness condition in Lemma~\ref{lem:bounded_tilting_via_event}, we identify the upper bound of the unnormalized density ratio:
\[
\frac{q(x)}{p_{\theta,0}(x)} \propto \exp\left(\frac{r(x)}{\tau}\right) \le \exp\left(\frac{r_{\max}}{\tau}\right) =C_q,
\]
where $r_{\max} \ge \sup_x r(x)$.
Following Lemma~\ref{lem:bounded_tilting_via_event}, we introduce an auxiliary variable $U \sim \mathrm{Unif}(0,1)$ and define the event:
\begin{align*}
    \mathcal{E}_{\bar X_0} = \left\{ U \le \frac{\exp(r(\bar X_0)/\tau)}{C_q} \right\} = \left\{ U \le \exp\left(\frac{r(\bar X_0) - r_{\max}}{\tau}\right) \right\}.
\end{align*}
The resulting $h$-function is then:
\begin{align*}
    h(x, 0) = \mathbb{P}(\mathcal{E}_{\bar X_0} \mid \bar X_0 = x)
            = \exp\left(\frac{r(x) - r_{\max}}{\tau}\right)
            \propto \exp\left(\frac{r(x)}{\tau}\right).
\end{align*}
This confirms that our choice of $h$-function corresponds to a valid conditioning event under the bounded reward assumption.
\subsection{Experiments on Improving Aesthetic Scores}
\label{app:exp_sd_1.5}

\begin{figure}[htbp]
    \centering
    \includegraphics[width=0.6\linewidth]{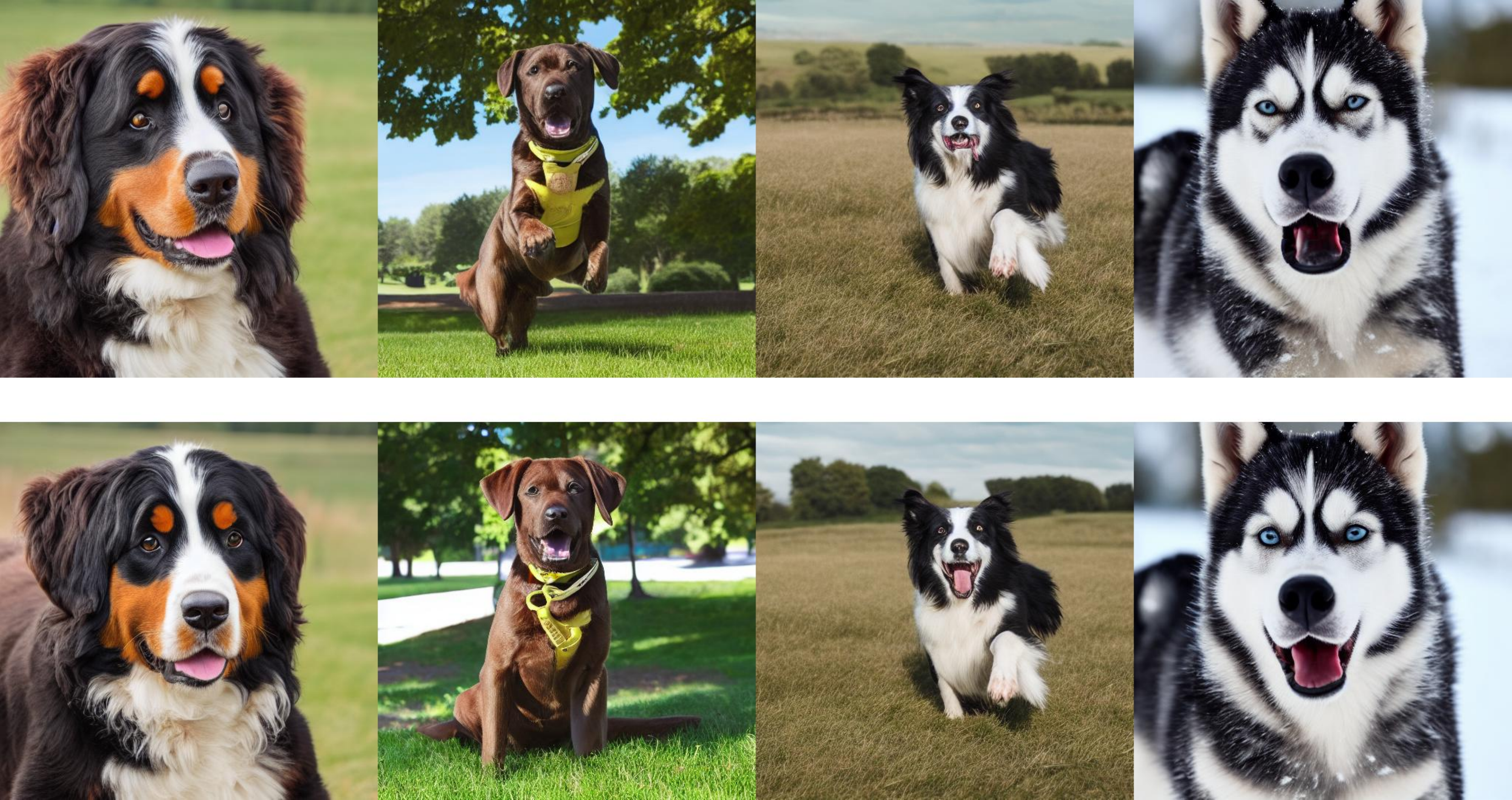}
    \caption{Comparison of dog images generated by \texttt{Stable Diffusion v1.5}. The upper row displays samples from vanilla generation process, while the bottom row shows samples guided by \algname{} using the aesthetic score as the reward.}
    \label{fig:demo}
\end{figure}

\subsubsection{Experimental Setup}
For the generation process, we utilize $L=20$ diffusion steps, employing the Euler ancestral sampler~\citep{karras2022elucidating}. To balance correction effectiveness with computational cost, Doob correction is applied only up to a cutoff timestep $l^*=10$. The approximation of $\nabla \log h$ is performed using $M=32$ Monte Carlo samples per step.

\subsubsection{Ablation Study on $\tau$ and $\gamma$}
We conduct an ablation study to analyze the impact of the temperature parameter $\tau$ and the strength $\gamma$. Tables~\ref{tab:min_score} through~\ref{tab:max_score} detail the statistical summaries of the generated aesthetic scores (minimum, quantiles, mean, and maximum) across various $(\tau, \gamma)$ configurations.

Our analysis reveals a trade-off between optimization intensity and stability. Specifically, the combination of a low temperature (small $\tau$) and high correction strength (large $\gamma$) results in aggressive optimization. However, overly aggressive settings (e.g., $\tau=0.2, \gamma=12$) can destabilize the mechanism, leading to potential collapse. Based on these empirical results, we identify a robust operating regime with $\tau \in [0.3, 0.4]$ and $\gamma \in [4.0, 8.0]$.

It is important to note that these optimal hyperparameters are specific to the current configuration. Altering the base model or sampler (e.g., using DDIM with stochasticity parameter $\eta > 0$) or the reward feedback may shift the stability region, necessitating further hyperparameter tuning.

\begin{table}[!htbp]
\centering
\small
\caption{Minimum aesthetic score.}
\begin{tabular}{c|cccc}
\hline
Doob $\gamma$ $\backslash$ $\tau$ & 0.2 & 0.3 & 0.4 & 0.5 \\
\hline
0.0 & 5.688 ± 0.142 & 5.688 ± 0.142 & 5.688 ± 0.142 & 5.688 ± 0.142 \\
1.0 & 5.728 ± 0.242 & 5.763 ± 0.116 & 5.733 ± 0.128 & 5.736 ± 0.137 \\
2.0 & 5.646 ± 0.417 & 5.740 ± 0.238 & 5.752 ± 0.152 & 5.741 ± 0.140 \\
3.0 & 5.782 ± 0.414 & 5.856 ± 0.227 & 5.760 ± 0.245 & 5.791 ± 0.167 \\
4.0 & 5.665 ± 0.345 & 5.757 ± 0.256 & 5.804 ± 0.218 & 5.794 ± 0.185 \\
5.0 & 5.004 ± 0.316 & 5.821 ± 0.226 & 5.877 ± 0.195 & 5.836 ± 0.263 \\
6.0 & 4.432 ± 0.159 & 5.921 ± 0.137 & 5.893 ± 0.176 & 5.847 ± 0.247 \\
7.0 & 3.856 ± 0.469 & 5.661 ± 0.416 & 5.870 ± 0.170 & 5.907 ± 0.188 \\
8.0 & 3.717 ± 0.239 & 5.480 ± 0.495 & 5.868 ± 0.209 & 5.840 ± 0.134 \\
9.0 & 3.094 ± 0.607 & 5.392 ± 0.177 & 5.971 ± 0.096 & 5.901 ± 0.159 \\
10.0 & 3.188 ± 0.216 & 4.389 ± 0.148 & 5.837 ± 0.052 & 5.638 ± 0.324 \\
11.0 & 3.296 ± 0.238 & 4.264 ± 0.465 & 5.763 ± 0.045 & 5.967 ± 0.073 \\
12.0 & 2.967 ± 0.137 & 4.157 ± 0.239 & 5.250 ± 0.351 & 5.948 ± 0.106 \\

\hline
\end{tabular}

\label{tab:min_score}
\end{table}

\begin{table}[!htbp]
\centering
\small
\caption{First quartile (Q1) of aesthetic score.}
\begin{tabular}{c|cccc}
\hline
Doob $\gamma$ $\backslash$ $\tau$ & 0.2 & 0.3 & 0.4 & 0.5 \\
\hline
0.0 & 6.157 ± 0.041 & 6.157 ± 0.041 & 6.157 ± 0.041 & 6.157 ± 0.041 \\
1.0 & 6.200 ± 0.050 & 6.199 ± 0.038 & 6.193 ± 0.038 & 6.188 ± 0.034 \\
2.0 & 6.270 ± 0.042 & 6.216 ± 0.040 & 6.209 ± 0.031 & 6.208 ± 0.026 \\
3.0 & 6.368 ± 0.019 & 6.282 ± 0.022 & 6.244 ± 0.032 & 6.235 ± 0.037 \\
4.0 & 6.433 ± 0.028 & 6.323 ± 0.017 & 6.272 ± 0.024 & 6.265 ± 0.031 \\
5.0 & 6.252 ± 0.101 & 6.393 ± 0.008 & 6.311 ± 0.011 & 6.282 ± 0.015 \\
6.0 & 5.816 ± 0.250 & 6.452 ± 0.028 & 6.361 ± 0.024 & 6.310 ± 0.012 \\
7.0 & 5.400 ± 0.119 & 6.453 ± 0.032 & 6.371 ± 0.012 & 6.319 ± 0.030 \\
8.0 & 4.920 ± 0.117 & 6.422 ± 0.036 & 6.452 ± 0.043 & 6.358 ± 0.011 \\
9.0 & 4.691 ± 0.104 & 6.210 ± 0.142 & 6.451 ± 0.029 & 6.388 ± 0.028 \\
10.0 & 4.385 ± 0.088 & 5.790 ± 0.223 & 6.484 ± 0.032 & 6.433 ± 0.049 \\
11.0 & 4.300 ± 0.067 & 5.477 ± 0.103 & 6.408 ± 0.040 & 6.471 ± 0.035 \\
12.0 & 3.989 ± 0.078 & 5.275 ± 0.144 & 6.255 ± 0.038 & 6.463 ± 0.032 \\

\hline
\end{tabular}

\label{tab:q1_score}
\end{table}

\begin{table}[!htbp]
\centering
\small
\caption{Mean aesthetic score.}
\begin{tabular}{c|cccc}
\hline
Doob $\gamma$ $\backslash$ $\tau$ & 0.2 & 0.3 & 0.4 & 0.5 \\
\hline
0.0 & 6.357 ± 0.029 & 6.357 ± 0.029 & 6.357 ± 0.029 & 6.357 ± 0.029 \\
1.0 & 6.452 ± 0.025 & 6.421 ± 0.025 & 6.403 ± 0.024 & 6.396 ± 0.026 \\
2.0 & 6.525 ± 0.033 & 6.472 ± 0.023 & 6.453 ± 0.024 & 6.434 ± 0.024 \\
3.0 & 6.600 ± 0.037 & 6.537 ± 0.028 & 6.490 ± 0.018 & 6.470 ± 0.023 \\
4.0 & 6.645 ± 0.022 & 6.576 ± 0.033 & 6.520 ± 0.028 & 6.498 ± 0.018 \\
5.0 & 6.476 ± 0.097 & 6.635 ± 0.036 & 6.569 ± 0.035 & 6.527 ± 0.023 \\
6.0 & 6.098 ± 0.123 & 6.691 ± 0.041 & 6.611 ± 0.042 & 6.559 ± 0.037 \\
7.0 & 5.736 ± 0.057 & 6.696 ± 0.077 & 6.652 ± 0.039 & 6.584 ± 0.035 \\
8.0 & 5.351 ± 0.099 & 6.631 ± 0.057 & 6.691 ± 0.050 & 6.625 ± 0.036 \\
9.0 & 5.042 ± 0.074 & 6.393 ± 0.089 & 6.702 ± 0.036 & 6.643 ± 0.042 \\
10.0 & 4.823 ± 0.082 & 6.054 ± 0.079 & 6.695 ± 0.040 & 6.671 ± 0.041 \\
11.0 & 4.698 ± 0.019 & 5.853 ± 0.062 & 6.596 ± 0.047 & 6.707 ± 0.036 \\
12.0 & 4.440 ± 0.066 & 5.674 ± 0.067 & 6.442 ± 0.049 & 6.701 ± 0.036 \\

\hline
\end{tabular}

\label{tab:mean_score}
\end{table}

\begin{table}[!htbp]
\centering
\small
\caption{Third quartile (Q3) of aesthetic score.}
\begin{tabular}{c|cccc}
\hline
Doob $\gamma$ $\backslash$ $\tau$ & 0.2 & 0.3 & 0.4 & 0.5 \\
\hline
0.0 & 6.560 ± 0.037 & 6.560 ± 0.037 & 6.560 ± 0.037 & 6.560 ± 0.037 \\
1.0 & 6.674 ± 0.027 & 6.635 ± 0.031 & 6.617 ± 0.030 & 6.596 ± 0.047 \\
2.0 & 6.787 ± 0.013 & 6.676 ± 0.023 & 6.655 ± 0.048 & 6.639 ± 0.051 \\
3.0 & 6.798 ± 0.080 & 6.736 ± 0.039 & 6.691 ± 0.031 & 6.679 ± 0.057 \\
4.0 & 6.852 ± 0.065 & 6.846 ± 0.070 & 6.728 ± 0.041 & 6.691 ± 0.035 \\
5.0 & 6.810 ± 0.054 & 6.857 ± 0.028 & 6.831 ± 0.054 & 6.719 ± 0.038 \\
6.0 & 6.536 ± 0.017 & 6.951 ± 0.039 & 6.852 ± 0.062 & 6.774 ± 0.034 \\
7.0 & 6.208 ± 0.084 & 6.897 ± 0.071 & 6.908 ± 0.013 & 6.816 ± 0.054 \\
8.0 & 5.860 ± 0.094 & 6.866 ± 0.062 & 6.925 ± 0.083 & 6.871 ± 0.059 \\
9.0 & 5.514 ± 0.090 & 6.655 ± 0.078 & 6.980 ± 0.100 & 6.884 ± 0.057 \\
10.0 & 5.290 ± 0.102 & 6.432 ± 0.070 & 6.950 ± 0.046 & 6.918 ± 0.101 \\
11.0 & 5.097 ± 0.050 & 6.286 ± 0.069 & 6.838 ± 0.086 & 6.909 ± 0.084 \\
12.0 & 4.907 ± 0.060 & 6.172 ± 0.060 & 6.732 ± 0.064 & 6.899 ± 0.027 \\

\hline
\end{tabular}

\label{tab:q3_score}
\end{table}

\begin{table}[!htbp]
\centering
\small
\caption{Maximum aesthetic score.}

\begin{tabular}{c|cccc}
\hline
Doob $\gamma$ $\backslash$ $\tau$ & 0.2 & 0.3 & 0.4 & 0.5 \\
\hline
0.0 & 7.028 ± 0.091 & 7.028 ± 0.091 & 7.028 ± 0.091 & 7.028 ± 0.091 \\
1.0 & 7.161 ± 0.139 & 7.125 ± 0.148 & 7.085 ± 0.111 & 7.078 ± 0.115 \\
2.0 & 7.291 ± 0.219 & 7.197 ± 0.185 & 7.164 ± 0.161 & 7.111 ± 0.165 \\
3.0 & 7.363 ± 0.192 & 7.296 ± 0.147 & 7.185 ± 0.169 & 7.183 ± 0.179 \\
4.0 & 7.292 ± 0.073 & 7.338 ± 0.187 & 7.251 ± 0.112 & 7.241 ± 0.156 \\
5.0 & 7.299 ± 0.125 & 7.385 ± 0.170 & 7.323 ± 0.133 & 7.288 ± 0.162 \\
6.0 & 6.966 ± 0.122 & 7.326 ± 0.156 & 7.320 ± 0.185 & 7.331 ± 0.210 \\
7.0 & 6.649 ± 0.154 & 7.407 ± 0.222 & 7.371 ± 0.161 & 7.302 ± 0.169 \\
8.0 & 6.699 ± 0.054 & 7.310 ± 0.094 & 7.498 ± 0.301 & 7.333 ± 0.161 \\
9.0 & 6.446 ± 0.065 & 7.131 ± 0.185 & 7.463 ± 0.292 & 7.299 ± 0.153 \\
10.0 & 6.210 ± 0.547 & 6.936 ± 0.093 & 7.273 ± 0.147 & 7.387 ± 0.064 \\
11.0 & 5.835 ± 0.069 & 6.850 ± 0.151 & 7.336 ± 0.081 & 7.314 ± 0.131 \\
12.0 & 5.516 ± 0.075 & 6.852 ± 0.025 & 7.179 ± 0.181 & 7.317 ± 0.175 \\
\hline
\end{tabular}
\label{tab:max_score}
\end{table}

\subsection{Experiments on Improving ImageReward with Reweighting and Resampling}
\label{app:exp_search}

Following \citet{zhang2025inference}, we build upon the official codebases of FK-Steering \citep{singhal2025general} and DAS \citep{kim2025test} to implement our baselines and sampling routines. We evaluate performance using the standard ImageReward prompt set \citep{singhal2025general}, utilizing $L=100$ total sampling steps and candidate set sizes of $K \in \{4, 8\}$. The sampler used in this experiment is chosen as DDIM~\citep{song2020denoising} with the stochasity parameter $\eta=1.0$. All results are reported as the mean and standard deviation over 4 independent trials.

For baseline configurations, we strictly adhere to the settings reported in \citet{zhang2025inference}, which replicate the original authors' implementations; we do not perform additional tuning for these baselines. Specifically, the BFS resampling schedule is defined by the interval $\mathcal{L} = [20, 80]$ with a frequency of $l_{\text{freq}} = 20$. We configure the BFS parameters as $\Omega_{\rm BFS} = (10, \mathtt{Increase}, \mathtt{Max}, \mathtt{SSP})$, corresponding to the search temperature, scoring method, buffer update strategy, and resampling algorithm, respectively. Regarding \algname{} -specific hyperparameters, we employ $M=32$ Monte Carlo samples, a cutoff time $l^*=20$, a temperature $\tau=0.2$, and a correction strength $\gamma=0.8$. Algorithm~\ref{alg:one-shot-stochastic-BFS} details the complete integration of our proposed method with the BFS framework.

\begin{algorithm}[t]
\caption{\algname{} + BFS Integration}
\label{alg:one-shot-stochastic-BFS}
\begin{algorithmic}[1]
\REQUIRE Pre-trained diffusion model $s_\theta(x,t)$; Doob correction strength $\gamma$;
Monte Carlo samples $M$;Time threshold $l^* \in [L]$;
Number of parallel trajectories $K$;
Resampling interval $\mathcal{L}$; Resampling frequency $l_{\rm freq}$; BFS parameters $\Omega_{\rm BFS}$.

\STATE Initialize resampling scores $\{b_{\mathrm{prev}}^{(k)}\}_{k=1}^K$.
\STATE Sample $\{x_{t_L}^{(k)}\}_{k=1}^K$ from $\mathcal{N}(0, I)$.

\FOR{$l = L, L-1, \dots, 1$}
    \IF{$l \in \mathcal{L}$ \textbf{and} $(l \pmod{l_{\rm freq}} == 0)$}
        \STATE $\{x_{t_l}^{(k)}\}_{k=1}^K, \{b_{\mathrm{prev}}^{(k)}\}_{k=1}^K \leftarrow \texttt{BFS}(\{x_{t_l}^{(k)}\}_{k=1}^K ; \Omega_{\mathrm{BFS}})$
    \ENDIF

    \FOR{$k=1, \dots, K$}
        \IF{$1<l \le l^*$}
        \STATE Sample $\{x_{t_{l-1}}^{(m)}\}_{m=1}^M$ from $\mathcal{N}\left(\mu_{t_l}(x_{t_l},s_\theta),\sigma^2_{t_l} I \right)$.
        \STATE Compute $\{\hat{x}_0^{(m)}\}_{m=1}^M$ via~\eqref{eq:x_0_surrogate}.
        \STATE Compute $\nabla \log \hat{h}(x^{(k)}_{t_l},t_l)$ in~\eqref{eq:mc_estimator} via $\{\hat{x}_0^{(m)}\}_{m=1}^M$.
            \STATE $\nabla \log \hat{p}_\theta^h(x_{t_l}^{(k)}) \leftarrow s_\theta(x_{t_l}^{(k)}, t_l) + \gamma\,\nabla \log \hat{h}(x_{t_l}^{(k)}, t_l)$.
            \STATE Sample $x_{t_{l-1}}$ from $ \mathcal{N}(\mu(x_{t_l},\nabla \log \hat{p}_\theta^h), \sigma^2_{t_l} I )$.
        \ELSE
            \STATE Sample $x_{t_{l-1}}$ from $\mathcal{N}(\mu(x_{t_l},s_\theta), \sigma^2_{t_l} I )$.
        \ENDIF
    \ENDFOR
\ENDFOR

\STATE \textbf{Return} $x_0 = \arg\max_{k \in \{1,\dots,K\}} r(x_{0}^{(k)})$.
\end{algorithmic}
\end{algorithm}

\subsection{Offline RL Experiments}
\label{app:exp_rl}

\paragraph{Setup.}
We adopt the experimental framework of \citet{lu2023contrastive}, utilizing their pre-trained diffusion policy and ground-truth Q-functions (as reward oracles). The diffusion model is conditioned on the state $s$ to generate actions $a$, employing a $L=15$ steps DDIM sampler.

We evaluate our method on the D4RL locomotion benchmark~\citep{fu2020d4rl}, which spans three MuJoCo environments  and three dataset compositions. Each environment represents a standard continuous-control task, while the datasets vary in the quality and diversity of the constituent trajectories.

\paragraph{Environments and Datasets.}
The specific tasks are defined as follows:
\begin{itemize}
    \item \textbf{HalfCheetah:} A planar running task that rewards maximizing forward velocity.
    \item \textbf{Hopper:} A one-legged hopping task requiring balance and forward progress.
    \item \textbf{Walker2d:} A bipedal walking task rewarding stable forward locomotion.
\end{itemize}

The datasets are categorized by the nature of the policy used for data collection:
\begin{itemize}
    \item \textbf{Medium-Expert:} A mixture of expert-level and medium-level policies’
decision data.
    \item \textbf{Medium:} Decision data generated by a single medium-level policy.
    \item \textbf{Medium-Replay:} Diverse decision data generated by a large set of medium-level policies.
\end{itemize}

\begin{table}[!htbp]
\centering
\small
\caption{Hyperparameter settings for our method across datasets and environments.
$\eta$ denotes the DDIM stochasticity parameter, $\gamma$ the Doob correction scale,
$\tau$ the temperature, $t^\star$ the correction start timestep,
and $K$ the number of candidate samples selected via best-of-$K$.}
\setlength{\tabcolsep}{6pt}
\renewcommand{\arraystretch}{1.15}
\begin{tabular}{llccccc}
\toprule
\textbf{Dataset} & \textbf{Environment} &
$\boldsymbol{\eta}$ &
$\boldsymbol{\gamma}$ &
$\boldsymbol{\tau}$ &
$\boldsymbol{t^\star}$ &
$\boldsymbol{K}$ \\
\midrule
Medium-Expert & HalfCheetah & 0.7 & 0.25 & 0.5 &  8  & 4 \\
Medium-Expert & Hopper      & 0.8 & 0.25 & 0.7 &  4  & 4 \\
Medium-Expert & Walker2d    & 0.4 & 0.5  & 0.4 & 10  & 8 \\
\midrule
Medium        & HalfCheetah & 0.2 & 0.25 & 0.5 & 10  & 4 \\
Medium        & Hopper      & 0.6 & 0.75 & 0.4 &  8  & 8 \\
Medium        & Walker2d    & 0.3 & 0.5  & 0.3 & 10  & 4 \\
\midrule
Medium-Replay & HalfCheetah & 0.4 & 0.5  & 0.3 & 10  & 8 \\
Medium-Replay & Hopper      & 0.2 & 0.5  & 0.5 &  6  & 4 \\
Medium-Replay & Walker2d    & 0.2 & 0.25 & 0.7 &  8  & 8 \\
\bottomrule
\end{tabular}

\label{tab:doob_hparams}
\end{table}

\paragraph{Hyperparameter Settings.}
To ensure a fair comparison, we align our evaluation protocol with \citet{zhang2025inference}, using their codebase and hyperparameter search strategy for baseline methods. A key metric for fairness is the computational budget, particularly for methods involving test-time search (sampling $K$ candidates and selecting the best-of-$K$).

The baselines employ various resource-intensive strategies: TFG~\citep{ye2024tfg} allows up to 8 recurrence steps; SVDD~\citep{li2024derivative} employs up to $K=16$ particles and $M=32$ Monte Carlo candidates, selecting the optimal temperature $\alpha$ from the set $\{0.0, 0.1, \dots, 0.6\}$; DAS~\citep{kim2025test} uses $K=16$ particles; and TTS~\citep{zhang2025inference} uses up to $K=4$ particles combined with optimized correction strength and iterative \& recurrence steps.
To maintain a comparable computational budget, we configure our method with $M=32$ lookahead samples (for gradient estimation) and generate up to $K=8$ final candidates. We note that unlike the baselines, our method does not require iterative recurrence or optimization steps during sampling.

We perform a hyperparameter search over the following grids: correction strength $\gamma \in \{0.25, 0.5, 0.75, 1\}$, temperature $\tau \in \{0.3,0.4,0.5,0.6,0.7\}$, and cutoff time $l^* \in \{4, 6, 8, 10\}$. Additionally, since \algname{}  and SVDD rely on the variance of the transition kernel (which vanishes in deterministic sampling), we tune the DDIM stochasticity parameter $\eta \in \{0.2,0.3,0.4,0.5,0.6,0.7,0.8\}$.  For fair comparison we evaluate our method on different seeds used for hyperparameter search. The optimal hyperparameters of \algname{} for each task are reported in Table~\ref{tab:doob_hparams}.

\end{document}